\pdfoutput=1
\documentclass[twoside]{article}
\usepackage{custom}
\usepackage[accepted]{aistats2021}
\usepackage[round]{natbib}

\begin{document}

\newif\ifappendix
\appendixfalse

%

%

\twocolumn[

\aistatstitle{Magnetic Manifold Hamiltonian Monte Carlo}

\aistatsauthor{ James A. Brofos \And Roy R. Lederman}

\aistatsaddress{ Yale University \And  Yale University} ]

\begin{abstract}
Markov chain Monte Carlo (MCMC) algorithms offer various strategies for sampling; the Hamiltonian Monte Carlo (HMC) family of samplers are MCMC algorithms which often exhibit improved mixing properties. The recently introduced magnetic HMC, a generalization of HMC motivated by the physics of particles influenced by magnetic field forces, has been demonstrated to improve the performance of HMC. In many applications, one wishes to sample from a distribution restricted to a constrained set, often manifested as an embedded manifold (for example, the surface of a sphere). We introduce magnetic manifold HMC, an HMC algorithm on embedded manifolds motivated by the physics of particles constrained to a manifold and moving under magnetic field forces. We discuss the theoretical properties of magnetic Hamiltonian dynamics on manifolds, and introduce a reversible and symplectic integrator for the HMC updates. We demonstrate that magnetic manifold HMC produces favorable sampling behaviors relative to the canonical variant of manifold-constrained HMC.
\end{abstract}

\section{INTRODUCTION}

Markov chain Monte Carlo (MCMC) is an important class of inference algorithm which has revolutionized inference in Bayesian statistical models. Originally developed by physicists, MCMC owes much to physical inspiration, including two popular techniques for Bayesian inference: the Metropolis-adjusted Langevin diffusion \citep{langevin-roberts-stramer} and Hamiltonian Monte Carlo (HMC) \citep{Duane1987216-1}. These methods may be straightforwardly applied to sample from differentiable densities on Euclidean spaces. For the history of MCMC, see \citep{mcmcrev,Robert_2011}.

Our purpose is to expand on the HMC literature for non-Euclidean spaces by continuing to draw on physics for inspiration. Examples of manifolds of interest that may be equipped with densities include the sphere, tori, the special orthogonal group ($n$-dimensional rotation matrices), and the Stiefel manifold ($n\times m$ matrices with orthogonal columns). Important contributions in these directions include \cite{Byrne_2013} and \cite{pmlr-v22-brubaker12} which develop symmetric and volume-preserving integrators based on closed-form geodesics and the generalized leapfrog algorithm, respectively.

However both of these methods, and virtually all HMC procedures besides, are based on canonical formalism of Hamiltonian dynamics from symplectic geometry. A notable exception is \cite{pmlr-v70-tripuraneni17a} which develops ``magnetic HMC'' for Euclidean spaces; magnetic HMC considers Markov chain transitions for Hamiltonian dynamics with magnetic field effects. In this work, we examine the foundations of Hamiltonian dynamics from the perspective of symplectic geometry on embedded manifolds. We propose a variant of Hamiltonian dynamics which does not conform to the canonical formalism. Instead, motion generated by these dynamics corresponds to the motion of a particle undergoing potential, magnetic field, and manifold constraint forces simultaneously. Using these dynamics as a transition mechanism, we formulate a method called magnetic manifold HMC. Although the underlying dynamics have a physical interpretation, an understanding of the physics is not critical for understanding magnetic manifold HMC as a sampler. {\it In our experimental evaluation, we show that magnetic manifold HMC produces better sampling behaviors in manifold-constrained inference tasks.} 
Magnetic manifold HMC has a degree of freedom in the choice of a magnetic structure. our experiments suggest that different choices of magnetic structures tend to recover different modes; therefore, multiple runs with different magnetic structures can be used to improve not only the local sampling properties of HMC, but also the exploration of different modalities of the posterior.

The outline of this paper is as follows. In \cref{sec:preliminaries} we examine key concepts such as symplecticness of a map, numerical integration, and Hamiltonian mechanics on embedded manifolds. In \cref{sec:hmc-on-manifold} we review the theory of Hamiltonian dynamics for use in a MCMC procedure for random variables that are constrained to a manifold. \Cref{sec:magnetic-manifold-hmc} introduces the magnetic manifold HMC sampler; we prove magnetic Hamiltonian mechanics on an embedded manifold conserve energy and the symplectic structure, and we introduce a symmetric and symplectic numerical integrator for magnetic dynamics on manifolds. In \cref{sec:experiments} we analyze the magnetic manifold HMC algorithm on inference tasks. \Cref{sec:conclusion} summarizes our contributions.

\section{PRELIMINARIES}\label{sec:preliminaries}

This section contains preliminary material to understand the majority of our paper. Topics include a construction of Hamiltonian mechanics, techniques and notions from embedded manifolds, and methods of numerical integration. The proofs in \cref{sec:magnetic-manifold-hmc} require methods from differential geometry; preliminary material for these are in \cref{app:geometry-preliminaries}. Throughout, we denote the $m\times m$ identity matrix by $\text{Id}_m$ and the $m\times m$ zero matrix by $\mathbf{0}_m$. The set of skew-symmetric $m\times m$ matrices is denoted $\text{Skew}(m)$.

\subsection{Embedded Manifolds}

In many cases, a manifold $M$ can be embedded in a Euclidean space $\R^m$ as the preimage of a constraint function $g : \R^m \to \R^k$ on the level set where $g$ takes the value zero: $M \defeq \set{q \in\R^m ~:~ g(q) = 0}$.
We denote the Jacobian of $g$ at the point $q\in M$ by $G(q)$. For any point $q\in M$, $G(q)$ is a $k \times m$ matrix. We assume that $G(q)$ is full-rank for any $q\in M$.
Many manifolds of interest may be written in this way such as the sphere, the special orthogonal group, the Stiefel manifold, and tori, among others. We define several important concepts related to embedded manifolds.

\begin{definition}[Tangent Space]
Let $q\in M$. The tangent space at $q$, denoted $\mathrm{T}_qM$, is the set of vectors satisfying,
\begin{align}
    \mathrm{T}_qM \defeq \set{\xi \in\R^m : G(q)\xi = 0}.
\end{align}
\end{definition}
where $G$ is the Jacobian of the constraint function $g$.
\begin{definition}[Cotangent Space]\label{def:cotangent-space}
The cotangent space at $q$, denoted $\mathrm{T}^*_qM$, is the set of vectors,
\begin{align}
    \mathrm{T}^*_qM \defeq \set{p \in\R^m : G(q)\nabla_p H(q, p) = 0},
\end{align}
where $H : \R^m\times\R^m \to\R$ is a smooth function. 
\end{definition}
The dependence of the cotangent space on the function $H$ is suppressed by convention. The tangent space is a vector space. When $G(q)\nabla_p H(q, p)$ is a linear function of $p$, the cotangent space is also a vector space.
\begin{definition}[Cotangent Bundle]
The set of vectors,
\begin{align}
    \mathrm{T}^*M \defeq \set{(q, p)\in\R^m\times\R^m ~:~ q\in M ~\text{and}~ p\in\mathrm{T}^*_qM}
\end{align}
is called the cotangent bundle.
\end{definition}
\begin{definition}\label{def:cotangent-embedding}
The embedding of $\mathrm{T}^*M$ in $\R^{2m}$ is defined to be the set of vectors
\begin{align}
    \set{(q,p)\in\R^{2m} : g(q) = 0 ~\text{and}~ G(q)\nabla_pH(q, p)=0}.
\end{align}
\end{definition}

\begin{definition}[Linear Maps between Tangent Spaces]
Let $M$ be a manifold. Let $\Phi : M\to M$ be a smooth function. Then $\mathrm{T}_q \Phi : \mathrm{T}_q M \to \mathrm{T}_{\Phi(q)} M$ is the linear mapping obtained by differentiating $\Phi$ at $q$. We use the notation $(\mathrm{T}_q \Phi)u$ to represent the linear map applied to $u$ yielding a vector in $\mathrm{T}_{\Phi(q)}M$. When $M$ is embedded in Euclidean space, $(\mathrm{T}_q\Phi) u = \nabla\Phi(q)^\top u$.
\end{definition}
\begin{definition}[Pullback]\label{def:pullback}
Given a map $\Omega : \mathrm{T}_q M\times \mathrm{T}_q M\to \R$, its pullback by a smooth function $\Phi : M\to M$ is the map $\Phi^*\Omega$ defined by $(\Phi^*\Omega)(u, v) \defeq \Omega((\mathrm{T}_q \Phi) u, (\mathrm{T}_q \Phi) v)$ where $u, v \in \mathrm{T}_qM$.
\end{definition}

\subsection{Hamiltonian Mechanics}

Hamiltonian mechanics are classically formulated as differential equations on the cotangent bundle of a smooth manifold $M$. Formally, given a manifold $M$, Hamiltonian mechanics give the time evolution of a point $(q, p)\in \mathrm{T}^*M$, often called phase-space in physics wherein $p$ is called the momentum. Recall that $\mathrm{T}^*M$ is an embedded manifold from \cref{def:cotangent-embedding}.

Our construction of Hamiltonian mechanics requires the specification of an object called the symplectic structure. One formulation of the symplectic structure uses a matrix associated with it.
Let $\mathbb{J}\in \text{Skew}(2m)$ be an invertible, skew-symmetric matrix.
Let $u,v\in\mathrm{T}_{(q, p)}\mathrm{T}^*M$ (i.e., two vectors, each in the tangent space to the cotangent space $\mathrm{T}^*M$, which is a manifold) with $u=(u_1,\ldots, u_{2m})$ and $v=(v_1,\ldots,v_{2m})$.
\begin{definition}[Symplectic Structure]\label{def:symplectic-structure}
The skew-symmetric, bilinear map $\Omega : \mathrm{T}_{(q, p)}\mathrm{T}^*M \times \mathrm{T}_{(q, p)}\mathrm{T}^*M \to \R$ defined by $\Omega(u, v) = u^\top \mathbb{J} v$ is called a symplectic structure on $\mathrm{T}^*M$ with matrix $\mathbb{J}$.
\end{definition}

\begin{definition}[Symplectic Transformation]\label{def:symplectic-transformation}
A map $\Phi:\mathrm{T}^*M\to\mathrm{T}^*M$ is symplectic if $\Phi^*\Omega=\Omega$, where $\Phi^*\Omega$ is the pullback (\cref{def:pullback}) of $\Omega$ by $\Phi$.
\end{definition}

Given a symplectic structure on $\mathrm{T}^*M$, we provide a definition of Hamilton's equations of motion.
\begin{definition}[Hamiltonian Vector Field]\label{def:hamiltonian-vector-field}
Let $\Omega$ be a symplectic structure on $\mathrm{T}^*M$ and let $H:\R^m\times\R^m\to\R$ be a smooth function; $H$ is called the Hamiltonian. Let $(q, p)\in\mathrm{T}^*M$ and let $\mathrm{T}_{(q,p)}\mathrm{T}^*M$ be the tangent space of $\mathrm{T}^*M$ at $(q, p)$. 
The unique Hamiltonian vector field $X_H : \mathrm{T}^*M \to \mathrm{T} \mathrm{T}^*M$ satisfies $\Omega(X_H(q, p), \delta) = (\mathrm{T}_{(q,p)} H) \delta$
for all $\delta\in\mathrm{T}_{(q,p)}\mathrm{T}^*M$.
\end{definition}

\begin{definition}[Hamiltonian Vector Field Flows]\label{def:vector-field-flow}
The flow of a Hamiltonian vector field $X_H : \mathrm{T}^*M\to\mathrm{T}\mathrm{T}^*M$ to time $t$ is the map $\Phi(\cdot,\cdot ; t) : \mathrm{T}^*M\to\mathrm{T}^*M$ satisfying $\frac{\mathrm{d}}{\mathrm{d}t} \Phi(q, p; t) = X_H(\Phi(q, p; t))$ and $\Phi(q, p; 0) = (q, p)$ for $(q, p)\in \mathrm{T}^*M$.
\end{definition}
\begin{definition}[Hamilton's Equations of Motion]\label{def:hamilton-equations-of-motion}
Suppose $(q_t, p_t) = \Phi(q, p; t)$. Since $\frac{\mathrm{d}}{\mathrm{d}t} \Phi(q, p; t) = X_H(\Phi(q, p; t))$, we have derived the equations of motion $(\dot{q}_t, \dot{p}_t) = X_H((q_t, p_t))$.
\end{definition}

The choice of symplectic form $\Omega$ affords a degree of freedom to Hamiltonian mechanics. 
The following example gives the form of $\Omega$ which recovers the canonical Hamiltonian equations of motion in Euclidean space.

\begin{example}
When $M\cong\R^m$, we have that $\mathrm{T}^* M\cong \R^{2m}$. The canonical symplectic structure $\Omega_\text{can}$ is a bilinear map from $\R^{2m}\times \R^{2m}$ to $\R$ with matrix
\begin{align}\label{eq:canonical-symplectic-matrix}
    \mathbb{J}_\text{can} = \begin{pmatrix} \mathbf{0}_m & \text{Id}_m \\ -\text{Id}_m &\mathbf{0}_m\end{pmatrix} \in \text{Skew}(2m).
\end{align}
such that $\Omega(\delta_1,\delta_2) = \delta_1^\top \mathbb{J}_\text{can} \delta_2$ for $\delta_1,\delta_2\in\R^{2m}$. \Cref{def:hamiltonian-vector-field} produces the familiar equations of motion $\dot{q}_t = \nabla_p H(q_t, p_t)$ and $\dot{p}_t = -\nabla_q H(q_t, p_t)$.
Hence the constraint $G(q)\nabla_p H(q, p) = 0$ in \cref{def:cotangent-embedding} means the velocity is constrained to the tangent space.
\end{example}

\subsection{Numerical Integration}

For most Hamiltonian vector fields, even those on Euclidean space, there do not exist closed-forms for the flows. Therefore, it is necessary to design numerical integrators for Hamiltonian systems.
\begin{definition}[Numerical Integrator]\label{def:numerical-integrator}
A numerical integrator of a Hamiltonian system with step-size $\epsilon\in\R$ and number of integration steps $N\in\mathbb{N}$ is a map $\hat{\Phi}(\cdot,\cdot; \epsilon, N) : \mathrm{T}^*M\to\mathrm{T}^*M$  approximating $\Phi(\cdot, \cdot; \epsilon \cdot N)$.
\end{definition}
While a good approximation is desirable in HMC for high acceptance probabilities, the quality of approximation is of no consequence for the correctness of the sampler. However, it is essential for our formulation of HMC that numerical integrators are {\it symmetric} and {\it symplectic}, defined as follows. 

\begin{definition}[Symmetric Map]\label{def:symmetric-map}
A map ${\Phi} : \mathrm{T}^*M\to\mathrm{T}^*M$ is symmetric if ${\Phi}({\Phi}(z; -\epsilon); \epsilon) = z$ for all $z\in \mathrm{T}^*M$.
\end{definition}

\begin{definition}[Symmetric Integrator]\label{def:symmetric-integrator}
A numerical integrator $\hat{\Phi} $ is symmetric if, for fixed $N$, $\hat{\Phi}(\cdot; \epsilon, N)$ is a symmetric map for all $\epsilon$.
\end{definition}

\begin{definition}[Symplectic Integrator]\label{def:symplectic-integrator}
A numerical integrator is symplectic if, for fixed $\epsilon$ and $N$, the map $\hat{\Phi}(\cdot; \epsilon, N)$ is symplectic (\cref{def:symplectic-transformation}).
\end{definition}
Symplectic integrators preserve volume in $\mathrm{T}^*M$ in the following sense; for details see \cref{app:conservation-of-volume}.
\begin{definition}[Volume Preserving]\label{def:volume-preserving}
A numerical integrator is volume preserving if for any region $Z\subset \mathrm{T}^*M$ with volume $\text{Vol}(Z)$ the set $Z' \defeq \{\hat{\Phi}(q, p; \epsilon, N) : (q,p)\in Z\}$ satisfies $\text{Vol}(Z)=\text{Vol}(Z')$ for any choice of $\epsilon$ and $N$.
\end{definition}
Flows of Hamiltonian vector fields are symmetric and symplectic; this fact, in combination with the technique of Strang splitting \citep{MacNamara2011OperatorS}, forms the basis of many symplectic integrators.

\section{RELATED WORK}

Our methodology is based on the HMC algorithm which is originally due to \cite{Duane1987216-1}. Two avenues of research are of immediate relevance to the present research. The first of these is research into non-canonical HMC, which explores non-canonical symplectic structures and their usefulness for inference. Magnetic HMC \citep{pmlr-v70-tripuraneni17a} is a special case of non-canonical HMC using a symplectic structure corresponding to motion of a particle in a magnetic field. 
Non-canonical HMC was further explored in \cite{brofoslederman2020noncanonical}, which proposed an explicit integration strategy for a broad class of non-canonical, constant symplectic structures.
The second avenue of research most related to our work is (canonical) HMC on manifolds. In \cite{doi:10.1111/j.1467-9868.2010.00765.x}, the authors consider inference on Riemannian manifolds with global coordinates. \cite{pmlr-v22-brubaker12} expands on this work by proposing an integrator suitable for embedded manifolds of Euclidean space via the method of Lagrange multipliers. An alternative approach was pursued in \cite{Byrne_2013} wherein the Lagrange multipliers are eliminated by formulating an intergrator using closed-form geodesics on embedded manifolds.

\section{HMC ON MANIFOLDS}\label{sec:hmc-on-manifold}

Let $M$ be a manifold embedded in Euclidean space. A probability density on $M$ is a map $\pi : M\to \R$ satisfying $\pi(q) \geq 0$ for all $q\in M$ and $\int_M \pi(q) ~\mathrm{d}q = 1$. We consider the case $\pi(q) \propto \exp(-U(q))$ where $U:\R^m \to\R$ is a smooth function called the potential energy. We consider Hamiltonians that may be expressed as the sum of the potential energy and another function $K:\R^m\to\R$ called the kinetic energy: $H(q, p) = U(q) + K(p)$. We restrict our attention to the case of a quadratic potential energy $K(p) = \frac{1}{2} p^\top p$. The Hamiltonians we consider are of the form
\begin{align}\label{eq:hamiltonian-form}
    H(q, p) = U(q) + \frac{1}{2} p^\top p.
\end{align}

Consider a joint distribution on $\mathrm{T}^*M$ defined by $\pi(q,p) \propto \exp(-H(q, p)) = \exp(-U(q)) \cdot \exp(- p^\top p / 2)$. We recognize the marginal distribution in $p$ (marginalizing out $q$) as a standard normal distribution subject to the constraint that $p\in \mathrm{T}_q^*M$. We give several definitions pertaining to MCMC methods on $\mathrm{T}^*M$.
\begin{definition}[Transition Operator]
The transition operator is a (possibly stochastic) map $\mathbb{Q} : \mathrm{T}^*M\to\mathrm{T}^*M$.
A Markov chain consists of repeated application of the transition operator.
\end{definition}
\begin{definition}[Transition Density]
The transition density $\Pi_{\mathbb{Q}}((q', p')\vert (q, p)) \in \R_+$ is the probability density that $\mathbb{Q}(q, p)$ equals $(q', p')$ given that the chain is currently in state $(q, p)\in \mathrm{T}^*M$.
\end{definition}
\begin{definition}[Stationary Distribution]
A distribution $\pi(q, p)$ is the stationary distribution of a Markov chain with transition density $\Pi_{\mathbb{Q}}$ if
\begin{align}\label{eq:stationary-distribution}
    \int_{\mathrm{T}^*M} \pi(q, p) \cdot \Pi_{\mathbb{Q}}((q', p')\vert (q, p))\mathrm{d}q\mathrm{d}p = \pi(q', p').
\end{align}
\end{definition}
\begin{definition}[Detailed Balance]
The transition operator $\mathbb{Q}$ satisfies detailed balance with respect to $\pi(q, p)$ if
\begin{align}\label{eq:detailed-balance}
    \pi(q, p) \cdot \Pi_\mathbb{Q}((q', p')\vert (q, p)) = \pi(q', p') \cdot\Pi_\mathbb{Q}((q, p)\vert(q', p')).
\end{align}
\end{definition}
The detailed balance condition says that, for the stationary distribution, the probability of being in state $(q, p)$ and transitioning to the state $(q', p')$ is equal to the probability of being in state $(q', p')$ and transitioning to the state $(q, p)$. If a Markov chain satisfies detailed balance with respect to $\pi(q, p)$, $\pi(q, p)$ is the stationary distribution of the chain, which is readily verified by substituting \cref{eq:detailed-balance} into \cref{eq:stationary-distribution}. For a discussion of conditions leading to the uniqueness of the stationary distribution, see \cite{10.5555/1051451}.

\subsection{Detailed Balance in HMC}

Symmetry and symplecticness are important to detailed balance in HMC. The following is reformulation of Theorem 1 from \cite{pmlr-v22-brubaker12}.
\begin{algorithm}[t!]
\caption{The transition operator for manifold-constrained Hamiltonian Monte Carlo Markov chain.}
\label{alg:mmhmc}
\begin{algorithmic}[1]
\State \textbf{Parameters}: Hamiltonian $H(q, p) = U(q) + \frac{1}{2} p^\top p$. Manifold $M = \set{q\in \R^m : g(q) = 0}$ embedded in $\R^m$. Symmetric and symplectic numerical integrator $\hat{\Phi}$.

\State \textbf{Input}: Initial position $q\in M$ and momentum $p\in \mathrm{T}^*_{q}M$. Number of integration steps $N\in\mathbb{N}$.
\State Sample $\epsilon \sim\text{DiscreteUniform}(\set{-\epsilon^*, +\epsilon^*})$.
\State Compute $(q', p') = \hat{\Phi}(q, p;\epsilon, N)$.
\State Sample $U\sim\text{Uniform}(0, 1)$.
\If{$U < \min\set{1, \exp(H(q, p) - H(q', p'))}$}
    \State \textbf{Return}: $(q', p')$.
\Else
    \State \textbf{Return}: $(q, p)$.
\EndIf
\end{algorithmic}
\end{algorithm}
\ifappendix
\begin{theorem*}
\else
\begin{theorem}\label{thm:detailed-balance}
\fi
Let $M = \set{q \in \R^m : g(q) = 0}$ be a connected manifold such that $G(q)$ has full-rank. Let $\text{T}^*M$ be an embedded sub-manifold of $\R^{2m}$ as in \cref{def:cotangent-embedding}. Let $q \in M$ and sample $p~\vert ~q \sim \text{Normal}(\mathbf{0}, \mathrm{Id}_m ~\vert~ G(q) p = 0)$. Let $H : \mathrm{T}^*M\to\R$ be a smooth Hamiltonian of the form in \cref{eq:hamiltonian-form}. 
Let $\hat{\Phi}$ be a symmetric (\cref{def:symmetric-integrator}) and symplectic (\cref{def:symplectic-integrator}) integrator.
Consider the transition operator $\mathbb{Q} : \mathrm{T}^*M \to\mathrm{T}^*M$ constructed in \cref{alg:mmhmc}.
The Markov chain with transition operator $\mathbb{Q}$ is stationary for the distribution $\pi(q,p) \propto e^{-H(q, p)}$.
\ifappendix
\end{theorem*}
\else
\end{theorem}
\fi
A proof is given in \cref{app:proof-of-detailed-balance}. We give the complete procedure for manifold-constrained HMC in \cref{alg:mmhmc}. To sample from $\pi(q)\propto \exp(-U(q))$, it suffices to project samples from $\pi(q, p)$ to their $q$-components. For details on HMC see \cite{10.5555/1162264}.

\subsection{Sampling in the Cotangent Space}

\Cref{thm:detailed-balance} requires sampling $p~\vert ~q \sim \text{Normal}(\mathbf{0}, \text{Id}_m ~\vert~ G(q) p = 0)$. It suffices to sample $p_\text{amb}\sim\text{Normal}(\mathbf{0}, \text{Id}_m)$ in the ambient Euclidean space and orthogonally project $p_\text{amb}$ to the cotangent space $\mathrm{T}^*_qM$.
For all of the manifolds we consider, there exists a closed-form for the orthogonal projection to the cotangent space. Formulas for orthogonal projections may be found in \cite{boumal2020intromanifolds}.

\section{MAGNETIC MANIFOLD HMC}\label{sec:magnetic-manifold-hmc}

This section formulates magnetic Hamiltonian mechanics on an embedded manifold. {\it We prove the dynamics are symmetric, symplectic and conserve energy. We propose a numerical integrator that is symmetric and symplectic, the two essential properties of integrators for HMC. As a consequence, we use this integrator in \cref{alg:mmhmc} to construct a Markov chain satisfying detailed balance with respect to the density $\pi(q, p)\propto \exp(-H(q, p))$ on $\mathrm{T}^*M$.} We define the symplectic structure corresponding to magnetic motion.

\begin{definition}\label{def:magnetic-symplectic-structure}
The magnetic symplectic structure, denoted $\Omega_\text{mag}$, is the symplectic structure with matrix
\begin{align}
    \mathbb{J}_\text{mag} = \begin{pmatrix} \mathrm{L} & \text{Id}_m \\ -\text{Id}_m & \mathbf{0}_m \end{pmatrix}
\end{align}
where $\mathrm{L}\in\text{Skew}(m)$.
\end{definition}
According to Dirac's theory of constraints \citep{Dirac:113811}, it suffices to embed a manifold-constrained Hamiltonian system in a Euclidean space. Consider the motion on $\mathrm{T}^*M$ determined by,
\begin{align}
    \dot{q}_t &= \nabla_p H(q_t, p_t) \label{eq:embedding-magnetic-velocity}\\
    \dot{p}_t &= -\nabla_q H(q_t, p_t) - \mathrm{L}\nabla_p H(q_t, p_t) - G(q_t)^\top \lambda \label{eq:embedding-magnetic-acceleration} \\
    g(q_t) &= 0 \label{eq:embedding-magnetic-constraint}
\end{align}
where $g : \R^m\to\R^k$ is a constraint function, $G : \R^m \to\R^{k\times m}$ is the Jacobian of the constraint, and $\lambda\in\R^k$ is a vector of Lagrange multipliers. The Lagrange multipliers $\lambda\equiv\lambda(q_t, p_t)$ are uniquely defined by the condition $g(q_t)=0$ along solutions of \cref{eq:embedding-magnetic-velocity,,eq:embedding-magnetic-acceleration,eq:embedding-magnetic-constraint}; see \cref{app:lagrange-multipliers}. These equations of motion corresponds to a distinct physical interpretation in $\R^3$; this motivates the name ``magnetic manifold HMC.'' See \cref{app:proof-of-lemma-forces}.
To formulate further results, we provide the following definition of a magnetic vector field flow.
\begin{definition}[Magnetic Vector Field Flow]\label{def:magnetic-flow}
Let $\Phi_\text{mag}(\cdot,\cdot; t) : \mathrm{T}^*M\to\mathrm{T}^*M$ be the vector field flow (see \cref{def:vector-field-flow}) to time $t$ corresponding to the motion given in \cref{eq:embedding-magnetic-velocity,,eq:embedding-magnetic-acceleration,eq:embedding-magnetic-constraint}.
\end{definition}

We proceed to give some theoretical results about magnetic Hamiltonian mechanics on a manifold embedded in Euclidean space. These results are a generalization of the proofs given for canonical dynamics in \cite{leimkuhler_reich_2005} to the case of magnetic dynamics and we seek to emulate their style.
An important fact of Hamiltonian dynamics is that their flows are symmetric, symplectic, and conserve the Hamiltonian; these properties hold for magnetic dynamics on a manifold embedded in Euclidean space.

\ifappendix
\begin{theorem*}
\else
\begin{theorem}\label{thm:magnetic-conservation}
\fi
Let $M = \set{q \in \R^m : g(q) = 0}$ be a connected manifold such that $G(q)$ has full-rank. Let $\text{T}^*M$ be an embedded sub-manifold of $\R^{2m}$ as in \cref{def:cotangent-embedding}. Let $\Omega_\text{mag}$ be the magnetic symplectic structure from \cref{def:magnetic-symplectic-structure} in the ambient Euclidean space $\R^{2m}\cong \R^m\times \R^m$. Let $H(q,p)$ be a smooth Hamiltonian $H:\R^m\times\R^m\to \R$ of the form in \cref{eq:hamiltonian-form}. Let $\Phi_\text{mag}$ be the magnetic vector field flow from \cref{def:magnetic-flow}. Then 
\begin{enumerate}[(i)]
\itemsep0em 
\item $\Phi_\text{mag}$ is a symmetric map (\cref{def:symmetric-map}): $\Phi_\text{mag}(\Phi_\text{mag}(q, p; t); -t) = (q, p)$.
\item  $\Phi_\text{mag}(\cdot,\cdot;t)$ is a symplectic transformation (\cref{def:symplectic-transformation}) on $\mathrm{T}^*M$: $\Phi_\text{mag}^*\Omega_\text{mag} = \Omega_\text{mag}$.
\item $\Phi_\text{mag}(\cdot,\cdot;t)$ conserves the Hamiltonian: $H(\Phi_\text{mag}(q, p;t)) = H(q, p)$ for any $(q,p)\in\mathrm{T}^*M$.
\end{enumerate}
\ifappendix
\end{theorem*}
\else
\end{theorem}
\fi

A proof is given in \cref{app:proof-magnetic-conservation}. 

\subsection{Numerical Integrators}

Remarkably, it is possible to give a {\it numerical} integrator that preserves properties (i) and (ii) from \cref{thm:magnetic-conservation} exactly. Numerical integration differs from exact Hamiltonian flows because they do not guarantee perfect conservation of the Hamiltonian (property (iii)).  HMC attempts to provide a good approximation in order to obtain high acceptance probabilities in the Metropolis accept-reject decision (the conditional statement at the end of \cref{alg:mmhmc}).

To develop the manifold-constrainted integrator, we require first an integrator for magnetic dynamics on Euclidean space. We will use a single-step subroutine contained in \cref{alg:euclidean-single-step}, which was proposed by \cite{pmlr-v70-tripuraneni17a} as a numerical integrator for magnetic dynamics in the case $M\cong \R^m$ and $\mathrm{T}^*M\cong\R^{2m}$. The single-step subroutine uses Strang splitting \citep{MacNamara2011OperatorS} to construct an integrator; see \cref{app:strang-splitting} for details on Strang splitting. Split a Hamiltonian in form of \cref{eq:hamiltonian-form} as $H(q, p) = H_1(q, p) + H_2(q, p) + H_1(q, p)$ where $H_1(q, p) = U(q) / 2$ and $H_2(q,p)= p^\top p/2$. Denote the magnetic vector field flows (\cref{def:magnetic-flow}) to time $\epsilon$ of $H_1$ and $H_2$ by $\Phi_1^\epsilon$ and $\Phi_2^\epsilon$, respectively.
For completeness, we restate in \cref{app:proof-euclidean-symmetric-symplectic} the closed-form expressions for $\Phi_1^\epsilon$ and $\Phi_2^\epsilon$, originally introduced in  \cite{pmlr-v70-tripuraneni17a}; see \cref{eq:euclidean-split-i,eq:euclidean-split-ii,eq:euclidean-split-ii-position,eq:euclidean-split-ii-momentum}, specifically.
The single-step subroutine (\cref{alg:euclidean-single-step}) computes the symmetric composition of Hamiltonian flows $\Phi_1^\epsilon\circ \Phi_2^\epsilon\circ\Phi_1^\epsilon :\R^{2m}\to\R^{2m}$; we note that this composition is defined on $\R^{2m}$ {\it and not specific to the manifold}. We require the following lemma.

\begin{algorithm}[t!]
\caption{The procedure for a single step of integrating Euclidean magnetic Hamiltonian trajectories. Closed-forms for $\Phi_1^\epsilon$ and $\Phi_2^\epsilon$ may be found in \cref{app:proof-euclidean-symmetric-symplectic}. This algorithm, and the closed-form flows $\Phi_1^\epsilon$ and $\Phi_2^\epsilon$, were derived in \cite{pmlr-v70-tripuraneni17a}.}
\label{alg:euclidean-single-step}
\begin{algorithmic}[1]
\State \textbf{Parameters}: Hamiltonian $H(q, p) = U(q) + \frac{1}{2} p^\top p$.
\State \textbf{Input}: Initial position and momentum variables $q_0\in \R^m$ and $p_0\in \R^m$. Integration step-size $\epsilon > 0$. Skew-symmetric matrix $\mathrm{L}\in\text{Skew}(m)$.
\State Compute $(q_0, p_{1/2}) = \Phi_1^\epsilon(q_0, p_0)$ where $\Phi_1^\epsilon$ is defined in \cref{eq:euclidean-split-i}.
\State Compute $(q_1, \bar{p}_{1/2}) = \Phi_2^\epsilon(q_0, p_{1/2}; \mathrm{L})$ where $\Phi_2^\epsilon$ is defined in \cref{eq:euclidean-split-ii,eq:euclidean-split-ii-position,eq:euclidean-split-ii-momentum}.
\State Compute $(q_1, p_1) = \Phi_1^\epsilon(q_1, \bar{p}_{1/2})$.
\State \textbf{Return}: $(q_1, p_1)$.
\end{algorithmic}
\end{algorithm}
\begin{algorithm}[t!]
\caption{The procedure for integrating manifold-constrained magnetic Hamiltonian trajectories. This is a symmetric and symplectic integrator. At each iteration, the pair $(q_{n+1}, p_{n+1})\in\mathrm{T}^*M$.}
\label{alg:manifold-integrator}
\begin{algorithmic}[1]
\State \textbf{Parameters}: Hamiltonian $H(q, p) = U(q) + \frac{1}{2} p^\top p$. Manifold $M = \set{q\in \R^m : g(q) = 0}$ embedded in $\R^m$ where $g:\R^m\to\R^k$. Jacobian of the constraint function $G:\R^m\to\R^{k\times m}$.
\State \textbf{Input}: Initial position and momentum variables $q_0\in M$ and $p_0\in \mathrm{T}^*_{q_0} M$. Integration step-size $\epsilon$ and number of integration steps $N\in\mathbb{N}$. Skew-symmetric matrix $\mathrm{L}\in\text{Skew}(m)$. 
\For{$n = 0,\ldots, N-1$}
    \State Compute $\mu$ using \cref{alg:lagrange-newton}.
    \State Compute $\bar{p}_{n+1/2}$ using \cref{eq:fibre-momentum}.
    \State Compute $(q_{n+1}, \bar{p}_{n+1})$ using \cref{alg:euclidean-single-step} with input $(q_{n}, \bar{p}_{n+1/2})$, step-size $\epsilon$, and $\mathrm{L}$.
    \State Use \cref{eq:lagrange-normal} to compute $\mu'$. 
    \State Compute $p_{n+1}$ using \cref{eq:cotangent-momentum}.
\EndFor
\State \textbf{Return}: $(q_N, p_N)$.
\end{algorithmic}
\end{algorithm}
\begin{algorithm}[t!]
\caption{Procedure for identifying the Lagrange multiplier $\mu$ that causes the position variable to satisfy the manifold constraint when integrated using \cref{alg:euclidean-single-step}.}
\label{alg:lagrange-newton}
\begin{algorithmic}[1]
\State \textbf{Parameters}: Hamiltonian $H(q, p) = U(q) + \frac{1}{2} p^\top p$. Manifold $M = \set{q\in \R^m : g(q) = 0}$ embedded in $\R^m$ where $g:\R^m\to\R^k$. Jacobian of the constraint function $G:\R^m\to\R^{k\times m}$.
\State \textbf{Input}: Initial position and momentum variables $q\in M$ and $p\in \mathrm{T}^*_{q} M$. Integration step-size $\epsilon$. Skew-symmetric matrix $\mathrm{L}\in\text{Skew}(m)$.
\State Let $(q'(\mu), p'(\mu))$ be the output of \cref{alg:euclidean-single-step} with input $(q, p - \frac{\epsilon}{2} G(q)^\top\mu)$, step-size $\epsilon$ and skew-symmetric matrix $\mathrm{L}$.
\State Define $f(\mu) = g(q'(\mu))$.
\State Find the root, $\mu^*$, of $f$ using Newton's method.
\State \textbf{Return}: $\mu^*$.
\end{algorithmic}
\end{algorithm}

\ifappendix
\begin{lemma*}[Symmetry and Symplecticness of \Cref{alg:euclidean-single-step}]
\else
\begin{lemma}[Symmetry and Symplecticness of \Cref{alg:euclidean-single-step}]\label{lemma:euclidean-symmetric-symplectic}
\fi
The single-step integrator for magnetic dynamics in Euclidean space in \cref{alg:euclidean-single-step} is symmetric and symplectic.
\ifappendix
\end{lemma*}
\else
\end{lemma}
\fi

A proof is given in \cref{app:proof-euclidean-symmetric-symplectic}. 
We propose a manifold-constrained integrator as the following series of updates. At iteration $n$, let $(q_n, p_n)\in\mathrm{T}^*M$. Compute:
\begin{align}
    \label{eq:fibre-momentum} \bar{p}_{n+1/2} &= p_n - \frac{\epsilon}{2} G(q_n)^\top \mu \\
    \label{eq:cotangent-position} (q_{n+1}, \bar{p}_{n+1}) &= \Phi_1^\epsilon\circ \Phi_2^\epsilon\circ\Phi_1^\epsilon(q_n, \bar{p}_{n+1/2}) \\
    \label{eq:position-constraint} 0 &= g(q_{n+1}) \\
    \label{eq:cotangent-momentum} p_{n+1} &= \bar{p}_{n+1} - \frac{\epsilon}{2} G(q_{n+1})^\top\mu' \\ 
    \label{eq:momentum-constraint} 0 &= G(q_{n+1})^\top p_{n+1}.
\end{align}
The Lagrange multipliers $\mu$ and $\mu'$ are chosen such that \cref{eq:position-constraint,eq:momentum-constraint} are satisfied. Such Lagrange multipliers exist, and are unique, provided $\epsilon \neq 0$ is small enough; see \cite{geometric-generalization}. Note that when \cref{eq:position-constraint,eq:momentum-constraint} are satisfied, $q_{n+1}\in M$ and $p_{n+1}\in\mathrm{T}^*_{q_{n+1}}M$ (for Hamiltonians of the form in \cref{eq:hamiltonian-form}) but that $\bar{p}_{n+1/2}$ and $\bar{p}_{n+1}$ are not guaranteed to respect the manifold constraint.

Pseudo-code for this manifold-constrained integrator corresponding to \cref{eq:fibre-momentum,eq:cotangent-position,eq:position-constraint,eq:cotangent-momentum,eq:momentum-constraint} is presented in \cref{alg:manifold-integrator}. To prove the manifold integrator in \cref{alg:manifold-integrator} is symmetric and symplectic, 
we leverage \cref{lemma:euclidean-symmetric-symplectic} to obtain the following result.

\ifappendix
\begin{theorem*}[Symmetry and Symplecticness of \Cref{alg:manifold-integrator}]
\else
\begin{theorem}[Symmetry and Symplecticness of \Cref{alg:manifold-integrator}]\label{thm:manifold-symmetric-symplectic}
\fi
The integration scheme in \cref{alg:manifold-integrator} is  symplectic and symmetric.
\ifappendix
\end{theorem*}
\else
\end{theorem}
\fi

A proof is given in \cref{app:proof-manifold-symmetric-symplectic}. The order  of \cref{alg:manifold-integrator} as an integrator of magnetic dynamics is derived in \cref{app:manifold-integrator-order}. Having constructed a symmetric and symplectic integrator on the manifold, we may apply \cref{alg:manifold-integrator} in \cref{alg:mmhmc} to yield a manifold-constrained magnetic HMC procedure.

{\bf Lagrange multipliers}. In practice, the Lagrange multiplier $\mu$ is identified via Newton's method. This procedure is summarized in \cref{alg:lagrange-newton}. Note that the root of $f$ in \cref{alg:lagrange-newton} is a Lagrange multiplier satisfying \cref{eq:position-constraint}. The second Lagrange multiplier $\mu'$ may be obtained in closed-form by rearranging \cref{eq:cotangent-momentum,eq:momentum-constraint} and solving the normal equations
\begin{align}\label{eq:lagrange-normal}
    \frac{\epsilon}{2} G(q_{n+1})G(q_{n+1})^\top \mu' = G(q_{n+1}) \bar{p}_{n+1},
\end{align}
which has a unique solution when $G(q_{n+1})$ has full rank (recall $G$ is the Jacobian of the constraint).

\section{EXPERIMENTS}\label{sec:experiments}

In this section we give experimental evaluations of the magnetic manifold HMC sampler. We compare against three competing methods: (i) canonical HMC on the manifold, (ii) Metropolis-adjusted Langevin diffusions on the manifold, and (iii) random walk Metropolis on the manifold, all of which were implemented according to the description in \cite{pmlr-v22-brubaker12}.
The magnetic structure $\mathrm{L}$ is a hyperparameter which we select by random search over five randomly-generated skew-symmetric matrices.

\subsection{Gaussian Under Linear Constraints}

\begin{table}[t!]
    \scriptsize
    \centering
    \caption{Minimum expected sample size (Min.) and mean expected sample size (Mean) metrics and per-second timing comparisons for the linearly-constrained Gaussian task over ten independent trials.}
    \begin{tabular}{l|rrrr}
        Method & Min. & Mean & Min. / Sec. & Mean / Sec. \\ \bottomrule
        Metropolis &  21.91 & 239.88 & 1,503.29 & 16,484.988 \\
        Langevin & 30.02 & 1,517.12 & 1,151.95 & 57,512.40 \\
        Canonical & 8,869.02 &  9,717.25 &  \bf{21,001.408} & 89,347.71 \\
        Magnetic & \bf{9,659.75} & \bf{9,914.93} & 20,768.29 & \bf{97,093.14}
    \end{tabular}
    \label{tab:lingauss}
\end{table}

Our first example considers sampling from a Gaussian distribution subject to linear constraints. Let $g(q) = Ax - b$ for $A\in\R^{m\times m}$ and $b\in \R^m$. $M= \set{q\in\R^m : g(q) =0 }$ is a linear submanifold of Euclidean space. We wish to draw samples from $\text{Normal}(\mu, \Sigma ~\vert~ g(q) = 0)$.
Following \cite{pmlr-v22-brubaker12} we set $b=(0,0)^\top$ and $A = \begin{pmatrix} 1  & 1 & 1 & 1 \\ 1 & 1 & -1 & 1 \end{pmatrix}$
and set the parameters of the normal distribution to be $\mu = (0,0,0,0)^\top$ and $\Sigma=\text{diag}(1, 1, 1/100, 1/100)$. We initialize each sampler at the mode of the distribution, which corresponds with the Gaussian mean at $\mu$. We sample 10,000 times from the target distribution and compute effective sample size statistics; we truncate the effective sample size at 10,000. We consider a grid of parameter values $\epsilon \in \set{1/10, 1/100, 1/1000}$ and $N\in\set{\mathrm{3, 5, 10, 100, 1000}}$. In computing the effective sample size, we report the best-case performance of canonical HMC, magnetic HMC, Metropolis-adjusted Langevin, and random walk Metropolis when results are averaged over ten independent trials of each parameter configuration. Results are shown in \cref{tab:lingauss}. We find that the magnetic integrator does best on the absolute measures of minimum and mean ESS, and mean ESS per second, but is worse than canonical HMC in minimum ESS per second.

\subsection{Bingham-von Mises-Fisher Distribution}

\begin{table}[t!]
    \scriptsize
    \centering
    \caption{Minimum expected sample size (Min.) and mean expected sample size (Mean) metrics and per-second timing comparisons for the Bingham-von Mises-Fisher task over ten independent trials.}
    \begin{tabular}{l|rrrr}
        Method & Min. & Mean & Min. / Sec. & Mean / Sec. \\ \bottomrule
        Metropolis & 370.013 & 445.074 & 2,541.308 & 3,057.195 \\
        Langevin & 879.768 & 1,076.740 & 4,850.082 & 5,936.750 \\
        Canonical & 6,330.063 & 8,284.166 & 6,528.653 & 9,488.570 \\
        Magnetic & \bf{10,000.0} & \bf{10,000.0} & \bf{11,268.489} &  \bf{11,268.489}
    \end{tabular}
    \label{tab:bvmf}
\end{table}

We next consider sampling from a Bingham-von Mises-Fisher distribution on $\mathbb{S}^5\subset\R^6$. This distribution is defined by $\pi(q) \propto\exp(b^\top q + q^\top Aq)$ for $b\in \R^6$ and $A\in \R^{6\times 6}$. We randomly generate a square positive definite matrix $A$ and standard normal vector $b$ and compare the performance of the four manifold samplers we consider. As in the linearly-constrained Gaussian experiments, we sample 10,000 times from the target distribution and compute effective sample size statistics; we truncate the effective sample size at 10,000. We consider a grid search over possible parameters $\epsilon\in\set{1/10, 1/100, 1/1000}$ and $N\in\set{5, 10, 100, 1000}$ and compute the best-case performance of the samplers when results are averaged over ten independent trials. Results are shown in \cref{tab:bvmf}. The magnetic integrator achieved an ESS of over 10,000 in each of the six coefficients, outperforming the other three samplers on this task; HMC can exhibit ESS exceeding the number of samples if samples are negatively correlated.

\subsection{Non-Conjugate Simplex Model}

\begin{table*}[t]
\scriptsize
\centering
\caption{Minimum expected sample size and per-second timing comparisons for the non-conjugate simplex task with varying Dirichlet parameterizations $\alpha$. Results averaged over fifty independent trials.}
\begin{tabular}{@{}rrrrrrr@{}}
\toprule
                                & \multicolumn{2}{c}{$\alpha=1$} & \multicolumn{2}{c}{$\alpha=3$} & \multicolumn{2}{c}{$\alpha=5$} \\ \midrule
\multicolumn{1}{l|}{}           & Min. ESS   & Min. ESS / Sec.   & Min. ESS   & Min. ESS / Sec.   & Min. ESS   & Min. ESS / Sec.   \\ \midrule
\multicolumn{1}{l|}{Metropolis} & 11.99      & 54.80             & 19.96    & 91.03             & 29.07      & 134.62            \\
\multicolumn{1}{l|}{Langevin}   & 13.39      & 19.32             & 23.36      & 34.63             & 34.96      & 52.53             \\
\multicolumn{1}{l|}{Canonical}  & {\bf 6987.37}    & {\bf 569.46}            & 7988.81      & 655.81            & 6601.39    & 545.24            \\
\multicolumn{1}{l|}{Magnetic}   & 6639.72    & 537.67            & {\bf 9893.07}    & {\bf 801.16}            & {\bf 9866.26}    & {\bf 822.97}            \\ \bottomrule
\end{tabular}
\label{tab:simplex}
\end{table*}
\begin{figure*}[t!]
    \centering
    \begin{subfigure}[b]{0.23\textwidth}
        \centering
        \includegraphics[width=\textwidth]{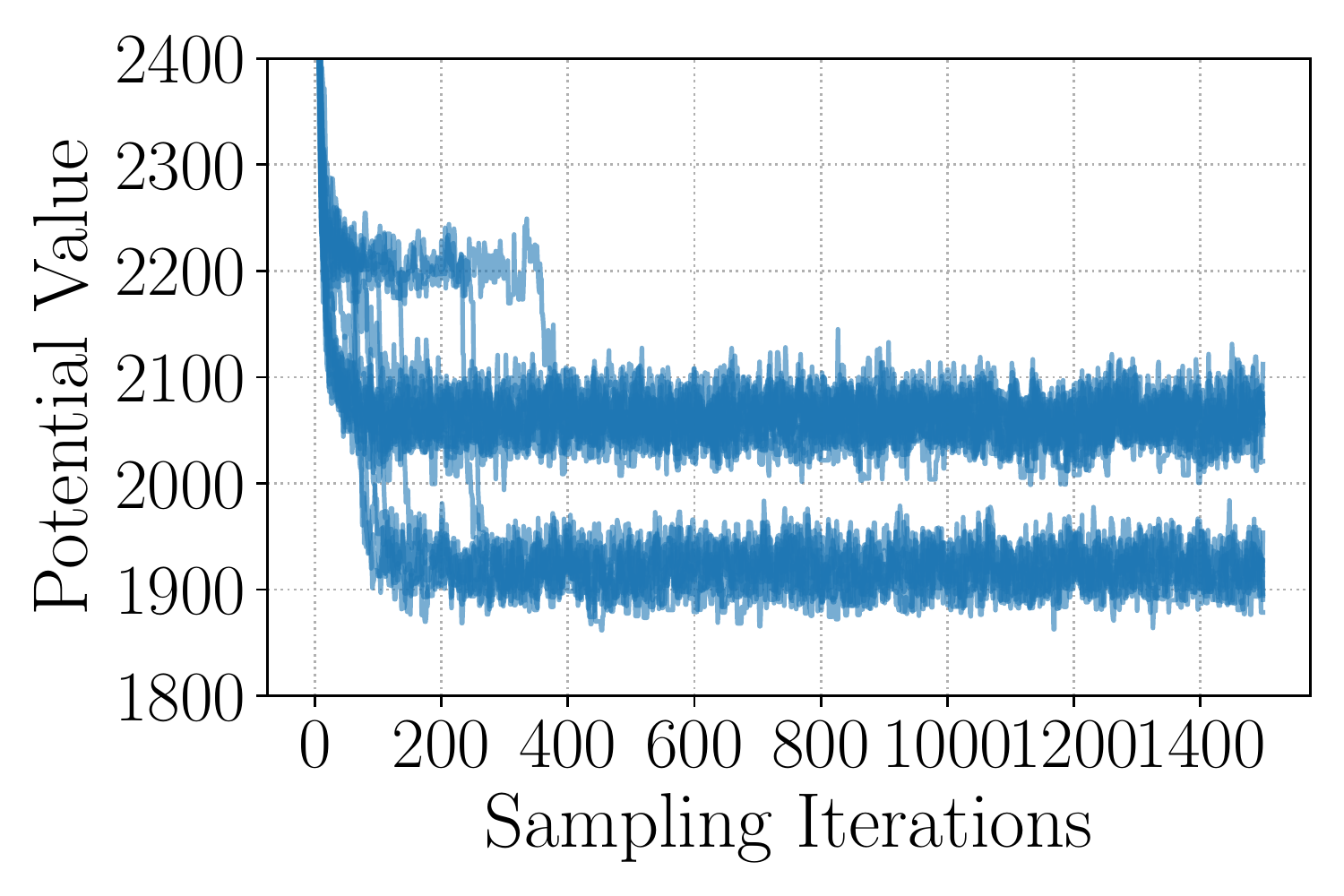}
        \caption{Canonical}
        \label{subfig:network-canonical}
    \end{subfigure}
    ~
    \begin{subfigure}[b]{0.23\textwidth}
        \centering
        \includegraphics[width=\textwidth]{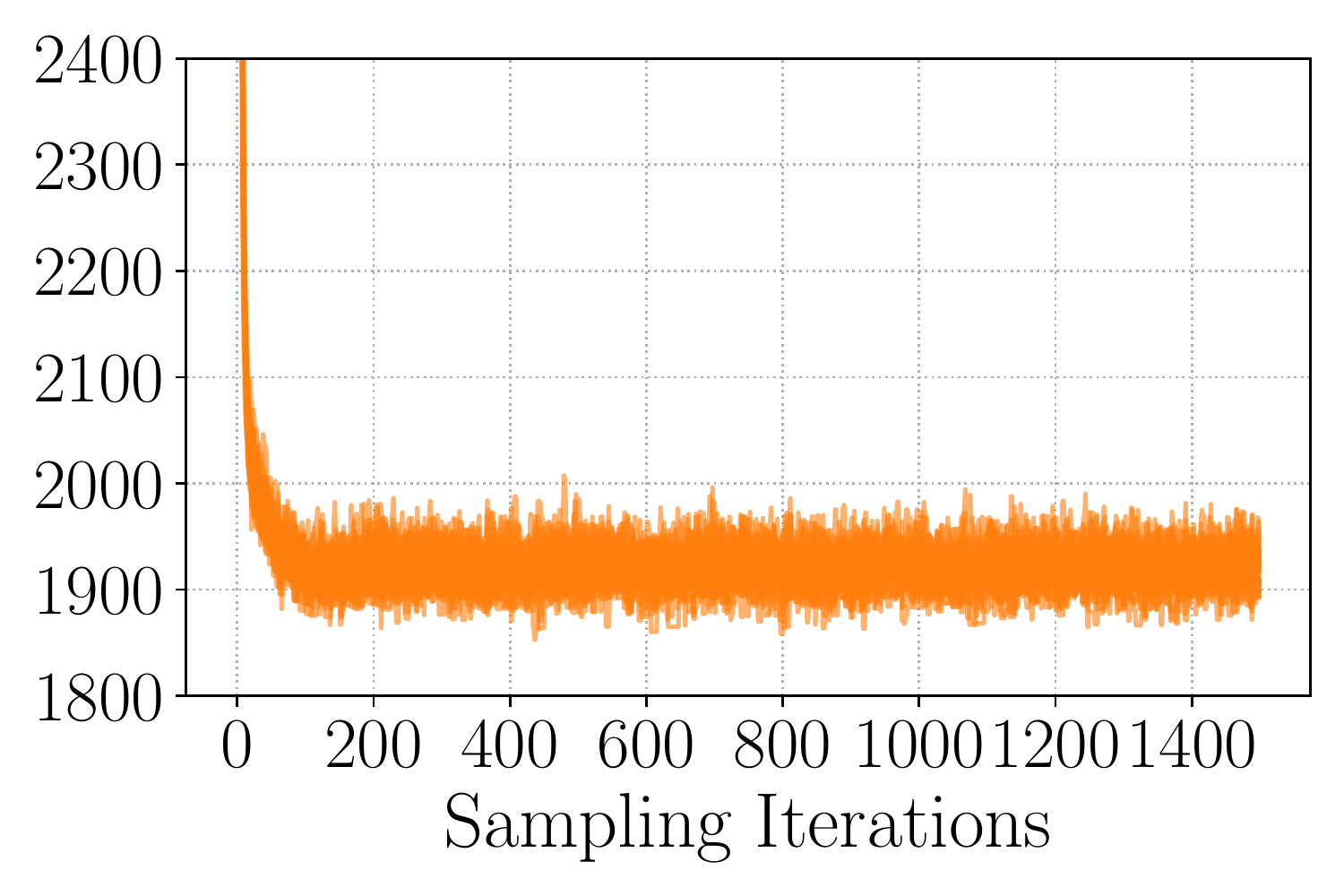}
        \caption{Large Mode Magnetic}
        \label{subfig:network-large}
    \end{subfigure}
    ~
    \begin{subfigure}[b]{0.23\textwidth}
        \centering
        \includegraphics[width=\textwidth]{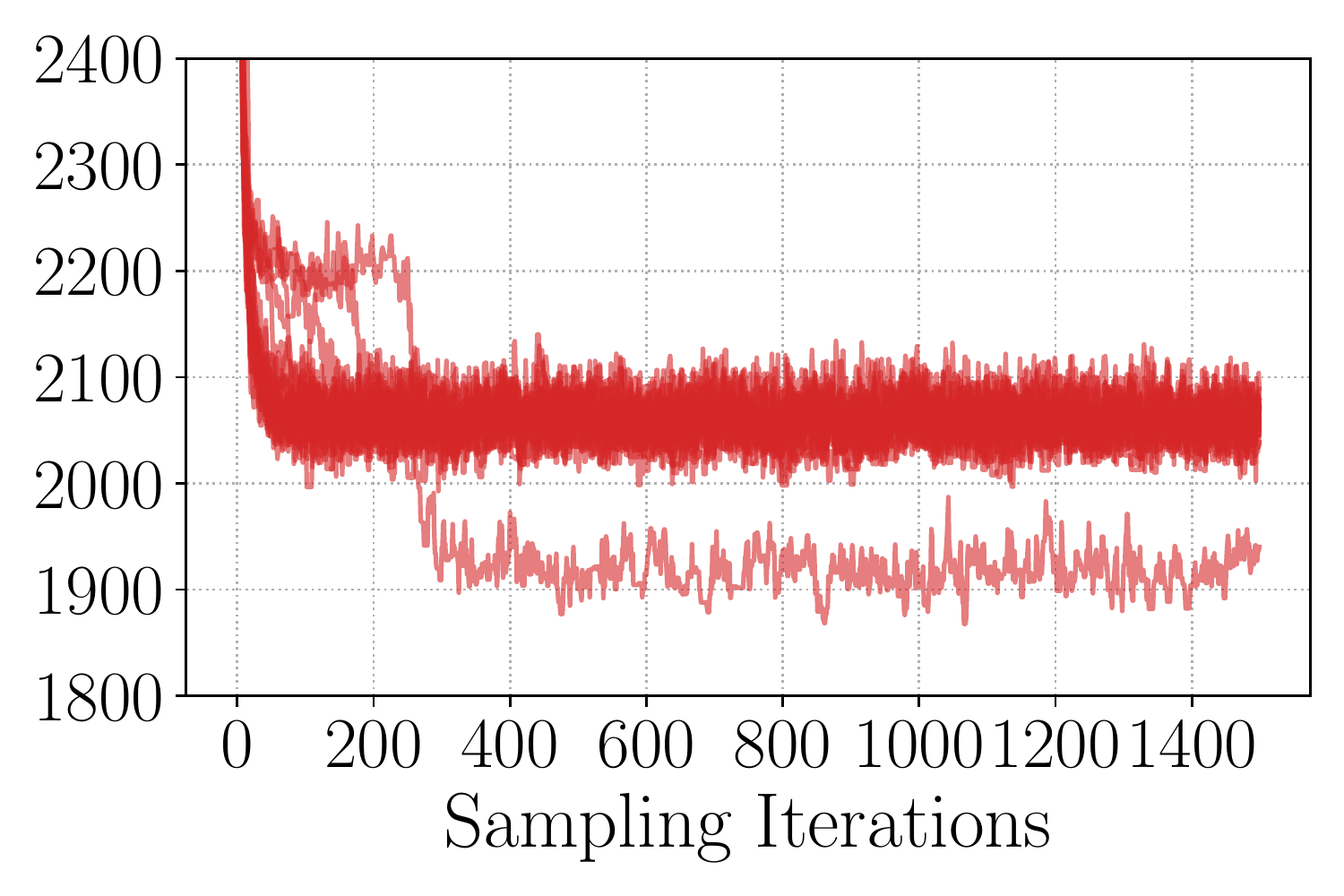}
        \caption{Small Mode Magnetic}
        \label{subfig:network-small}
    \end{subfigure}
    ~
    \begin{subfigure}[b]{0.23\textwidth}
        \centering
        \includegraphics[width=\textwidth]{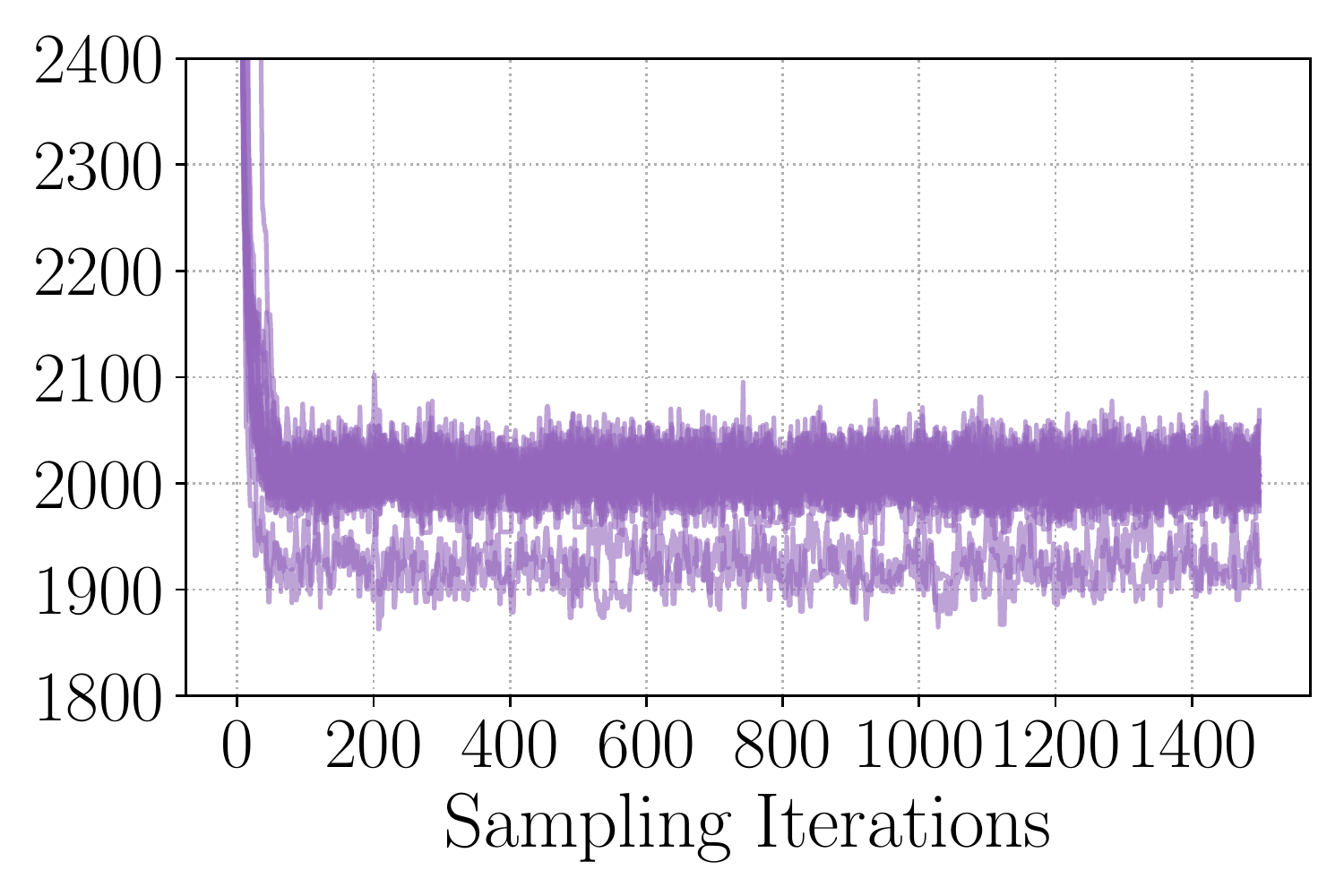}
        \caption{Rare Mode Magnetic}
        \label{subfig:network-rare}
    \end{subfigure}
    \caption{Twenty canonical and magnetic sampling trajectories on the network eigenstructure task. Magnetic dynamics can be encouraged to sample particular modes of the posterior depending on the magnetic structure. \Cref{subfig:network-canonical} is implemented using the description in \cite{pmlr-v22-brubaker12}. The magnetic manifold HMC has the matrix $\mathrm{L}$ as a parameter.
    We sampled ten different matrices; the mode-finding behavior of magnetic manifold HMC for three different choices of $\mathrm{L}$ is shown in \Cref{subfig:network-large,,subfig:network-small,subfig:network-rare}.}
    \label{fig:network-modes}
\end{figure*}

Denote the simplex embedded in $\R^n$ by $\Delta^{n-1} = \set{\theta \in\R^n : \theta \geq 0 ~\text{and}~ \sum_{i=1}^n \theta_i = 1}$. We consider the volleyball dataset from \cite{hyper2}, which consists of nine volleyball players; each player has a skill $\theta_i$ such that $(\theta_1,\ldots,\theta_9)\in\Delta^8$. For each game, players are partitioned into teams $T_1,T_2\subset\set{1,\ldots, 9}$ and the probability that $T_1$ triumphs over $T_2$ in a game of volleyball is modeled as $\sum_{i\in T_1} \theta_i / \sum_{i\in T_1\cup T_2} \theta_i$. Given a Dirichlet prior $\theta \sim \text{Dirichlet}(\alpha_1,\ldots,\alpha_9)$ and observations of teams and victories, the inference task is to sample from the posterior over $\theta$. Because the simplex is not expressible as the zero levelset of some function (due the the positivity constraint), we follow \cite{Byrne_2013} and embed the simplex into the positive orthant of $\mathbb{S}^{n-1}$ by the mapping $(q_1,\ldots, q_n) \defeq (\sqrt{\theta_1},\ldots,\sqrt{\theta_n})$ and draw samples on $\mathbb{S}^{n-1}$ instead of $\Delta^{n-1}$; see \cite{Byrne_2013} for full details. Samples on $\mathbb{S}^{n-1}$ can be transformed back to $\Delta^{n-1}$ by the map $q_i \mapsto q_i^2$. We set $\epsilon = 0.01$ and $N=20$ in these experiments and consider $\alpha_1=\ldots=\alpha_9=\alpha$ for $\alpha\in\set{1, 3, 5}$. Results are summarized in \cref{tab:simplex}; magnetic manifold HMC is strongest when an informative ($\alpha > 1$) Dirichlet prior is used whereas canonical HMC performs better in the case of the non-informative prior $\alpha=1$.

\subsection{Network Eigenmodel}

We consider Bayesian inference in the context of network analysis using the example from \cite{Byrne_2013,doi:10.1198/jcgs.2009.07177}. This application considers protein interactions in a network of 230 proteins. Formally, the observations consist of a $230\times 230$ adjacency matrix $\Delta$ whose $(i,j)$ entry, $\delta_{ij}$, equals one if the $i^\text{th}$ and $j^\text{th}$ proteins interact. Let $\phi : \R\to(0,1)$ denote the probit function. The objective is to perform inference in the following Bayesian model: 
\begin{align}
    \delta_{ij} ~\vert~ U, \Sigma, c &\sim\text{Bernoulli}(\phi((U\Sigma U^\top)_{ij} + c))
\end{align}
with priors $\sigma_i \sim \text{Normal}(0, 230)$, $c \sim\text{Normal}(0, 100)$, and $U \sim\text{Uniform}(\text{Stiefel}(230, 3))$,
where $\Sigma = \text{diag}(\sigma_1,\sigma_2, \sigma_3)$ and $\text{Stiefel}(230, 3)$ is the Stiefel manifold consisting of $230\times 3$ orthogonal matrices. This model is interpreted as identifying a low-rank eigendecomposition of a matrix whose probit transform models the probability of proteins interacting.

This is a challenging posterior for gradient-based Bayesian inference because it is multi-modal. To sample from all the modes of the distribution, it is necessary to combine multiple Markov chains with a parallel tempering scheme \citep{Byrne_2013}. This method permits HMC transitions to go between modes of the distribution. Our experiments instead consider the question of whether particular choices of magnetic structure influence which mode of the distribution magnetic manifold HMC will target. We therefore generated several magnetic structures by skew-symmetrizing a standard normal matrix and compared their mode-finding behavior.

We found that random walk Metropolis and manifold Langevin were ineffective in this task. Therefore, we restrict our discussion to the canonical and magnetic variants of HMC. Let $\Xi~\text{diag}(\alpha)~\Xi^\top$ be the rank-3 singular value decomposition of $\Delta$ where $\Xi\in\text{Stiefel}(230, 3)$ and $\alpha\in\R^3$. From the initial condition $c=0$, $\sigma_i=\alpha_i$ for $i=1,2,3$, and $U=\Xi$, canonical HMC regularly falls into one of two modes with potential values approximately 2,100 (smaller mode) and 1,900 (larger mode), a phenomenon previously observed in \citep{Byrne_2013}. Intriguingly, it is possible to prescribe magnetic structures which tend to target either of these modes, with virtually all sampling trajectories of magnetic HMC concentrating in one of the modes. Even more interesting is that there is a magnetic structure which sometimes targets a ``rare mode'' with potential value approximately 2,000 that canonical HMC never enters. All of these phenomena are illustrated across twenty random sampling trajectories in \cref{fig:network-modes}. This phenomenon could be exploited for targeting modes in multi-modal distributions.

\section{CONCLUSION}\label{sec:conclusion}

This paper presented the magnetic manifold HMC algorithm. We discussed the theory of magnetic Hamiltonian dynamics embedded in an ambient Euclidean space. We proved that these dynamics conserve energy and volume on the manifold. We proposed a symmetric and symplectic numerical integrator for these dynamics. We evaluated the magnetic manifold HMC procedure on manifold-constrained sampling tasks. Our experimental results show the promise of introducing magnetic effects into the proposal operator used in HMC.
We defer to future work the study of how magnetic structures may be generated to explore the posterior and favor certain modalities.

\section*{Acknowledgments}

This material is based upon work supported by the National Science Foundation Graduate Research Fellowship under Grant No. 1752134. Any opinion, findings, and conclusions or recommendations expressed in this material are those of the authors and do not necessarily reflect the views of the National Science Foundation. James Brofos's affiliation with The MITRE Corporation is provided for identification purposes only, and is not intended to convey or imply MITRE's concurrence with, or support for, the positions, opinions, or viewpoints expressed by the author.

\bibliography{thebib}

\newpage
\onecolumn
\appendix
\appendixtrue

\newpage
\section{Extended Preliminaries}\label{app:geometry-preliminaries}

This appendix is intended to provide preliminary mathematics, geometry, and physics for understanding the proofs. It is organized as a collection of definitions and facts which are referenced in the proofs where they are needed. Where possible, citations with page numbers are given for previously established facts.

\subsection{General Mathematics}

\begin{definition}[Permutation Group]
Denote by $S^n$ the group of permutations on $n$ elements.
\end{definition}
\begin{definition}
The sign of a permutation $\sigma\in S^n$, denoted $\text{sign}(\sigma)$, is the parity ($+1$ if even, $-1$ if odd) of the number of transpositions required to write the permutation.
\end{definition}

\begin{definition}[Skew-Symmetric Matrix]\label{def:skew-symmetric}
A matrix $\mathbb{J}$ is skew-symmetric if $\mathbb{J}^\top = -\mathbb{J}$. The set of skew-symmetric $n\times n$ matrices is denoted $\text{Skew}(n)$.
\end{definition}
\begin{fact}[Skew-Symmetric Matrices Annihilate Vectors]\label{fact:skew-symmetric-annihilation}
For a skew-symmetric matrix $\mathbb{J}\in\text{Skew}(n)$ and a vector $x\in\R^n$, $x^\top \mathbb{J} x = 0$.
\end{fact}

\begin{definition}[Skew-Symmetric Linear Map (Page 393 in \cite{10.5555/50877})]
Let $V$ be a vector space. A map 
\begin{align}
\alpha : \underbrace{V\times\cdots\times V}_{k~\text{times}}\to\R
\end{align}
is skew-symmetric if
\begin{align}
    \alpha(v_1,\ldots,v_k) = \text{sign}(\sigma) \alpha(v_{\sigma(1)}, \ldots, v_{\sigma(k)})
\end{align}
where $\sigma\in S^k$ is a permutation.
\end{definition}

\begin{fact}[Inverse Function Theorem]\label{fact:inverse-function-theorem}
Let $f:\R^k\to\R^k$ be differentiable. Suppose that at $\mu\in\R^k$ the Jacobian $\nabla_\mu f(\mu)$ has non-zero determinant. Then there exists an open set $O$ containing $\mu$ such that there exists a differentiable inverse function $f^{-1} : f(O) \to O$.
\end{fact}

\subsection{Differential Forms}

Differential forms are an important topic in differential geometry. Nearly any book on differential geometry will contain a detailed discussion of these objects. For instance, \cite{lee2003introduction,10.5555/1965128,10.5555/50877} all contain detailed sections on differential forms.

\begin{definition}[Differential $k$-form (Page 129 of \cite{10.5555/1965128})]
Let $M$ be a manifold of dimension $m$ and let $q\in M$. A differential $k$-form $\alpha$ on $M$ ($k\leq m$) is a skew-symmetric linear map
\begin{align}
    \alpha : \underbrace{\mathrm{T}_qM\times\cdots\times\mathrm{T}_qM}_{k ~\text{times}}\to\R.
\end{align}
\end{definition}

For a complete appreciation of our theoretical results, an understanding of 1-, 2-, and $m$-forms will be required. The most important 1-forms are the coordinate 1-forms.
\begin{definition}[Coordinate 1-Forms]\label{def:coordinate-1-form}
Let $M$ be a manifold of dimension $m$ and let $q\in M$ with $q = (q_1,\ldots, q_m)$. The coordinate 1-forms are $\mathrm{d}q_i : \mathrm{T}_qM \to \R$ defined by $\mathrm{d}q_i(v) = v_i$ where $v = (v_1,\ldots, v_m)\in\mathrm{T}_qM$.
\end{definition}

The wedge product of differential forms is the principle tool by which differential forms are combined to give another differential form.
\begin{definition}[Wedge Product]
Let $\alpha$ be a differential $k$-form and $\beta$ a differential $l$-form. The wedge product of $\alpha$ and $\beta$, denoted $\alpha\wedge\beta$, is a $(k+l)$-form defined by
\begin{align}
    (\alpha\wedge\beta)(v_1,\ldots, v_{k+l}) \defeq \frac{1}{k!l!} \sum_{\sigma \in S^{k+l}} \text{sign}(\sigma) \alpha(v_{\sigma(1)},\ldots, v_{\sigma(k)}) \beta(v_{\sigma(k+1)},\ldots,v_{\sigma(k+l)})
\end{align}
where $S^{k+l}$ denotes the permutation group on $k+l$ elements. 
\end{definition}

\begin{fact}[Wedge Product of Coordinate 1-Forms]\label{fact:wedge-product-cooordinate-1-form}
Let $\mathrm{d}q_i$ and $\mathrm{d}q_j$ be coordinate 1-forms. Let $u,v\in\mathrm{T}_qM$ with $u=(u_1,\ldots,u_m)$ and $v=(v_1,\ldots, v_m)$. Then
\begin{align}
    (\mathrm{d}q_i\wedge \mathrm{d}q_j)(u, v) = u_iv_j - u_jv_i
\end{align}
\end{fact}
\begin{proof}
Using \cref{def:coordinate-1-form} and the fact that the permutation group on two elements has only two elements, direct computation yields,
\begin{align}
    (\mathrm{d}q_i\wedge \mathrm{d}q_j)(u, v) &= \frac{1}{1!1!} \mathrm{d}q_i(u)\mathrm{d}q_j(v) - \mathrm{d}q_i(v)\mathrm{d}q_j(u) \\
    &= u_iv_j - v_iu_j \\
    &= u_iv_j - u_jv_i.
\end{align}
\end{proof}

\begin{fact}[$k$-Forms from Coordinate 1-Forms (Page 131 in \cite{10.5555/1965128})]\label{fact:k-form-from-1-form}
In terms of the coordinate 1-forms, any differential $k$-form may be written as
\begin{align}
    \alpha = \sum_{i_1 < \cdots < i_k} \alpha_{i_1,\ldots,i_k}(q) ~~\mathrm{d}q_{i_1}\wedge \cdots \wedge \mathrm{d}q_{i_k}
\end{align}
where $\alpha_{i_1,\ldots,i_k} : M \to \R$ are smooth functions.
\end{fact}

\begin{definition}[Constant Differential Form]\label{def:constant-k-form}
A differential $k$-form is called constant when, for all $i_1<\cdots<i_k$, the $\alpha_{i_1,\ldots,i_k}$ in \cref{fact:k-form-from-1-form} are all constant functions.
\end{definition}

\begin{fact}[Wedge Product and Pullback (Page 131 in \cite{10.5555/1965128})]\label{fact:wedge-product-pullback}
Let $\alpha$ be differential $k$-form and $\beta$ be a differential $l$-form on a manifold $M$. Let $\Phi:M\to M$ be a smooth function. Then,
\begin{align}
    \Phi^* (\alpha \wedge\beta) = (\Phi^* \alpha) \wedge (\Phi^*\beta)
\end{align}
\end{fact}

\begin{definition}[Non-Vanishing Differential Form]\label{def:non-vanishing}
A differential $k$-form $\alpha$ is said to be non-vanishing if for every $q\in M$ there exists $v_1,\ldots,v_k\in \mathrm{T}_qM$ such that $\alpha(v_1,\ldots, v_k)\neq 0$.
\end{definition}
\begin{definition}[Volume Form (Page 139 in \cite{10.5555/1965128})]\label{def:volume-form}
Given a manifold $M$ of dimension $m$, a nowhere vanishing differential $m$-form on $M$ is called a volume form.
\end{definition}

\begin{fact}[Dimension of Volume Forms (Page 399 in \cite{10.5555/50877})]\label{fact:dimension-of-volume-form}
The vector space of all constant $m$-forms on $\R^{m}$ is a vector space of dimension one.
\end{fact}

\begin{definition}[Determinant]\label{def:determinant}
Let $\Phi : M\to M$ be a smooth map and $V$ a volume form on $M$. Then $\Phi^*V$ is another $m$-form on $M$. The function $\text{det}(\Phi) : M\to\R$ such that
\begin{align}
    \Phi^* V = \text{det}(\Phi) V
\end{align}
is called the determinant of $\Phi$.
\end{definition}
\begin{fact}[Volume Preservation and Determinant (Page 140 in \cite{10.5555/1965128})]\label{fact:determinant}
A transformation $\Phi$ is volume preserving for $V$ if and only if $\text{det}(\Phi) = 1$.
\end{fact}

It will be convenient to work with vectors of differential 1-forms rather than individual 1-forms. The following definition extends the wedge product of differential 1-forms to vectors of differential 1-forms.

\begin{definition}[Wedge Product of Vectors of 1-Forms]\label{def:wedge-product-vectors}
Let $\mathrm{d}a$ and $\mathrm{d}b$ be $m$-dimensional vectors of differential 1-forms. For instance $\mathrm{d}a = (\mathrm{d}a_1,\ldots, \mathrm{d}a_m)$. The wedge product of such vectors is defined by the relation 
\begin{align}
\mathrm{d}a\wedge\mathrm{d}b \defeq \sum_{i=1}^m \mathrm{d}a_i \wedge\mathrm{d}b_i
\end{align}
\end{definition}

\begin{fact}[Properties of Wedge Product (Page 64 in \cite{leimkuhler_reich_2005})]\label{fact:properties-wedge-product}
Let $\mathrm{d}a$, $\mathrm{d}b$, $\mathrm{d}c$ be $m$-dimensional vectors of differential 1-forms. For instance $\mathrm{d}a = (\mathrm{d}a_1,\ldots, \mathrm{d}a_m)$. The following are properties of the wedge product:
\begin{enumerate}
    \item Skew-symmetry:
    \begin{align}\label{eq:wedge-skew-symmetry}
    \mathrm{d}a \wedge \mathrm{d}b = -\mathrm{d}b \wedge \mathrm{d}a
    \end{align}
    \item Linearity:
    \begin{align}\label{eq:wedge-linearity}
    \mathrm{d}a \wedge (r ~\mathrm{d}b \wedge s~\mathrm{d}c) = r ~\mathrm{d}a \wedge \mathrm{d}b + s~ \mathrm{d}a\wedge\mathrm{d}c
    \end{align}
    for $r,s\in\R$.
    \item Matrix multiplication: For a matrix $\mathrm{L}\in\R^{m\times m}$,
    \begin{align}\label{eq:wedge-matrix}
    \mathrm{d}a\wedge \mathrm{L}~\mathrm{d}b = \mathrm{L}^\top~\mathrm{d}a\wedge\mathrm{d}b.
    \end{align}
    \item Annihilation: When $\mathrm{L}$ is a symmetric matrix, 
    \begin{align}\label{eq:wedge-annihilation}
    \mathrm{d}a\wedge\mathrm{L}~\mathrm{d}a = 0.
    \end{align}
\end{enumerate}
\end{fact}

\begin{fact}[Differential 2-forms and Symplectic Structures]\label{fact:2-form-symplectic-structure}
Let $q = (q_1,\ldots,q_m)\in M$ and $p=(p_1,\ldots,p_m)\in \mathrm{T}_q^*M$ and set $z = (q, p) \in \mathrm{T}^*M$. A symplectic structure $\Omega$ (see \cref{def:symplectic-structure}) with matrix $\mathbb{J} \in \text{Skew}(2m)$ may be written in terms of wedge products as
\begin{align}
    \Omega &= \sum_{i<j} \mathbb{J}_{ij} ~\mathrm{d}z_i\wedge\mathrm{d}z_j \label{eq:symplectic-wedge-i} \\
    &= \frac{1}{2} ~\mathrm{d}z \wedge  \mathbb{J}\mathrm{d}z \label{eq:symplectic-wedge-ii}
\end{align}
\end{fact}
\begin{proof}
Let $u, v\in \mathrm{T}_z\mathrm{T}^*M$. The relation \cref{eq:symplectic-wedge-i} is standard and may be found in \cite{10.5555/1965128} on page 147. To prove it, it suffices to use \cref{def:coordinate-1-form} and \cref{fact:wedge-product-cooordinate-1-form} which yields
\begin{align}
    \sum_{i<j} \mathbb{J}_{ij} ~\mathrm{d}z_i\wedge\mathrm{d}z_j(u, v) &= \sum_{i<j} \mathbb{J}_{ij} u_iv_j - \mathbb{J}_{ij} u_jv_i \\
    &= \sum_{i<j} \mathbb{J}_{ij} u_iv_j + \mathbb{J}_{ji} u_jv_i \\
    &= \sum_{i=1}^{2m} \sum_{j=1}^{2m} \mathbb{J}_{ij} u_iv_j \\
    &= u^\top\mathbb{J}v \\
    &= \Omega(u, v)
\end{align}

\Cref{eq:symplectic-wedge-ii} follows first from
\begin{align}
    \sum_{i=1}^{2m}\sum_{j=1}^{2m} \mathbb{J}_{ij} \mathrm{d}z_i \wedge\mathrm{d}z_j  &=\sum_{i<j} \mathbb{J}_{ij} \mathrm{d}z_i \wedge\mathrm{d}z_j + \mathbb{J}_{ji} \mathrm{d}z_j \wedge\mathrm{d}z_i \\
    &= \sum_{i<j} \mathbb{J}_{ij} \mathrm{d}z_i \wedge\mathrm{d}z_j - \mathbb{J}_{ij} \mathrm{d}z_j \wedge\mathrm{d}z_i \\
    &= \sum_{i<j} \mathbb{J}_{ij} \mathrm{d}z_i \wedge\mathrm{d}z_j -  \mathrm{d}z_j \wedge \mathbb{J}_{ij}\mathrm{d}z_i \\
    &= \sum_{i<j} \mathbb{J}_{ij} \mathrm{d}z_i \wedge\mathrm{d}z_j +  \mathbb{J}_{ij}\mathrm{d}z_i\wedge \mathrm{d}z_j \\
    &= \sum_{i<j} 2 ~\mathbb{J}_{ij} \mathrm{d}z_i \wedge\mathrm{d}z_j
\end{align}
and, using \cref{def:wedge-product-vectors}, from
\begin{align}
    \sum_{i<j} \mathbb{J}_{ij} \mathrm{d}z_i \wedge\mathrm{d}z_j &= \frac{1}{2}\sum_{i=1}^{2m}\sum_{j=1}^{2m} \mathbb{J}_{ij} \mathrm{d}z_i \wedge\mathrm{d}z_j \\
    &= \frac{1}{2}\sum_{i=1}^{2m}\sum_{j=1}^{2m} \mathrm{d}z_i \wedge \mathbb{J}_{ij} \mathrm{d}z_j \\
    &= \frac{1}{2}\sum_{i=1}^{2m} \mathrm{d}z_i \wedge \sum_{j=1}^{2m} \mathbb{J}_{ij} \mathrm{d}z_j \\
    &= \frac{1}{2} \mathrm{d}z \wedge \mathbb{J}\mathrm{d}z
\end{align}
\end{proof}

\begin{fact}[Constant Symplectic Structure]
The symplectic structures we consider are constant (see \cref{def:constant-k-form}) since $\mathbb{J}_{ij}$ does not depend on $z$.
\end{fact}

\begin{fact}[Magnetic Symplectic Structure]\label{fact:magnetic-symplectic-structure}
In the particular case corresponding to a magnetic symplectic structure we will have
\begin{align}
    \mathbb{J}_\text{mag} = \begin{pmatrix} \mathrm{L} & \text{Id}_m \\ -\text{Id}_m & \mathbf{0}_m \end{pmatrix} \in \text{Skew}(2m)
\end{align}
for some skew-symmetric matrix $\mathrm{L}\in\R^{m\times m}$. Applying \cref{def:wedge-product-vectors,fact:2-form-symplectic-structure}, the symplectic form can be expressed as
\begin{align}
    \sum_{i<j} \mathbb{J}_{ij} \mathrm{d}z_i \wedge\mathrm{d}z_j &= \sum_{i=1}^n \mathrm{d}q_i\wedge\mathrm{d}p_i + \sum_{i<j} \mathrm{L}_{ij} \mathrm{d}q_i\wedge \mathrm{d}q_j \\
    &= \mathrm{d}q \wedge\mathrm{d}p + \frac{1}{2} \mathrm{d}q \wedge \mathrm{L}~\mathrm{d}q.\label{eq:magnetic-wedge-product}
\end{align}
\end{fact}

\begin{fact}[Magnetic Motion]\label{fact:magnetic-motion}
For a Hamiltonian $H : \R^m\times\R^m\to\R$, the motion of $(q, p)\in\R^{2m}$ under a magnetic symplectic structure is given by
\begin{align}
    \dot{q} &= \nabla_pH(q, p) \\
    \dot{p} &= -\nabla_qH(q, p) - \mathrm{L}\nabla_pH(q, p) 
\end{align}
\end{fact}
\begin{proof}
Identify $z=(q, p)$. Given a magnetic symplectic structure with matrix,
\begin{align}
    \mathbb{J}_\text{mag} = \begin{pmatrix} \mathrm{L} & \text{Id}_m \\ -\text{Id}_m & \mathbf{0}_m \end{pmatrix}
\end{align}
for $\mathrm{L}\in\text{Skew}(m)$, the Hamiltonian vector field $X_H$ is defined by
\begin{align}
    & \Omega_\text{mag}(X_H(z), \delta) = \nabla_z H(z)^\top \delta \\
    \implies& X_H(z)^\top \mathbb{J}_\text{mag} \delta = \nabla_z H(z)^\top \delta \\
    \implies& \mathbb{J}_\text{mag}^\top X_H(z) = \nabla_zH(z) \\
    \implies& X_H(z) = (-\mathbb{J}_\text{mag})^{-1} \nabla_zH(z)
\end{align}
where we have used that $\mathbb{J}_\text{mag}$ is skew-symmetric and therefore satisfies $\mathbb{J}_\text{mag}^\top = -\mathbb{J}_\text{mag}$ from \cref{def:skew-symmetric}. Moreover, the inverse of $-\mathbb{J}_\text{mag}$ is
\begin{align}
    (-\mathbb{J}_\text{mag})^{-1} = \begin{pmatrix} \mathbf{0}_m & \text{Id}_m \\ -\text{Id}_m & -\mathrm{L} \end{pmatrix}.
\end{align}
Therefore,
\begin{align}\label{eq:magnetic-equations-of-motion}
    \begin{pmatrix} \dot{q} \\ \dot{p} \end{pmatrix} = \begin{pmatrix} \mathbf{0}_m & \text{Id}_m \\ -\text{Id}_m & -\mathrm{L} \end{pmatrix} \begin{pmatrix} \nabla_qH(q, p) \\ \nabla_pH(q, p) \end{pmatrix}
\end{align}
is the Hamiltonian vector field.
\end{proof}

An important volume form for Hamiltonian mechanics is the Liouville volume form, which is constructed from differential 2-forms.
\begin{definition}[Liouville Volume Form (Page 149 in \cite{10.5555/1965128})]\label{def:liouville-volume-form}
Let $M$ be a manifold of dimenion $m$ and let $\Omega$ be a symplectic 2-form on $M$. The Liouville volume form on $\mathrm{T}^*M$ is the $2m$-form defined by,
\begin{align}
    \Lambda \defeq \frac{(-1)^{m(m-1)/2}}{m!} \Omega\wedge \cdots\wedge\Omega ~~~~~~\text{(there are $m$ copies of $\Omega$ in the wedge products)}.
\end{align}
\end{definition}
When $\Omega = \Omega_\text{can} = \mathrm{d}q\wedge\mathrm{d}p$, denote the Liouville volume form by $\Lambda_\text{can}$. The Liouville volume form $\Lambda_\text{can}(v_1,\ldots, v_{2m})$ with $v_i = (v_{i,1},\ldots, v_{i,2m})$ is proportional to the determinant of the matrix whose $(i,j)$ entry is $v_i^j$, which, in turn, is the signed volume of parallelpiped spanned by the columns of that matrix.

\begin{definition}[Diffeomorphism of $\mathrm{T}^*M$]\label{def:diffeomorphism}
Let $\Phi : \mathrm{T}^*M\to\mathrm{T}^*M$ be a smooth, invertible mapping. Then $\Phi$ is called a diffeomorphism of $\mathrm{T}^*M$
\end{definition}
\begin{fact}[Differential Forms and Change-of-Variables (Page 62 in \cite{leimkuhler_reich_2005})]\label{fact:change-of-variables}
Let $M$ be a manifold of dimension $m$ with $z\in \mathrm{T}^*M$. Let $\mathrm{d}z$ be the vector of coordinate 1-forms; see \cref{def:coordinate-1-form}. Let $\Phi$ be a smooth function and let $\hat{z} = \Phi(z)$. Then the coordinate 1-forms of $\hat{z}$ are transformations of the coordinate 1-forms of $z$:
\begin{align}
    \mathrm{d}\hat{z}_i &= \sum_{j=1}^{2m} \frac{\partial \hat{z}_i}{\partial z_j} \mathrm{d}z_j \\
    &= \sum_{j=1}^{2m} \frac{\partial \Phi(z)_i}{\partial z_j} \mathrm{d}z_j
\end{align}
Or, letting $\mathrm{d}z = (\mathrm{d}z_1,\ldots,\mathrm{d}z_{2m})$,
\begin{align}
    \mathrm{d}\hat{z} = \nabla_z\Phi(z)^\top \mathrm{d}z
\end{align}
where $\nabla_z\Phi(z)^\top$ is the Jacobian of $\Phi$.
\end{fact}

\begin{fact}[Symplecticness and Differential 2-Forms]\label{fact:symplecticness-and-2-forms}
Let $\Omega$ be a symplectic structure with matrix $\mathbb{J}$. A map $\Phi:\mathrm{T}^*M\to\mathrm{T}^*M$ is symplectic with respect to $\Omega$ if and only if
\begin{align}
    \frac{1}{2} \mathrm{d}\hat{z}\wedge\mathbb{J}\mathrm{d}\hat{z} = \frac{1}{2} \mathrm{d}z\wedge\mathbb{J}\mathrm{d}z
\end{align}
where $\mathrm{d}\hat{z} = \nabla_z\Phi(z)^\top \mathrm{d}z$.
\end{fact}

\begin{proof}
A symplectic tranformation is one that preserves the symplectic structure under pullback. If $\Phi:\mathrm{T}^*M\to\mathrm{T}^*M$ then
\begin{align}
    (\Phi^*\Omega)(u, v) \defeq \Omega((\mathrm{T}_z\Phi) u, (\mathrm{T}_z\Phi) v) = \Omega(u, v) \iff \Phi~\text{is symplectic}
\end{align}
for all $u, v\in\mathrm{T}_z\mathrm{T}^*M$. Letting $u=(u_1,\ldots, u_{2m})$ and $v=(v_1,\ldots, v_{2m})$, in terms of the matrix $\mathbb{J}$, this is nothing but
\begin{align}
    (\nabla_z\Phi(z)^\top u)^\top \mathbb{J} (\nabla_z\Phi(z)^\top v) = u^\top \mathbb{J}v
\end{align}
or
\begin{align}
    \nabla_z\Phi(z)\mathbb{J}\nabla_z\Phi(z)^\top = \mathbb{J}.
\end{align}
We can now establish that if $\hat{z} = \Phi(z)$ then symplecticness of $\Phi$ is equivalent to conservation of the 2-form. From \cref{fact:2-form-symplectic-structure}, $\Omega$ can be written in terms of the wedge product as,
\begin{align}
    \Omega = \frac{1}{2} \mathrm{d}z \wedge\mathbb{J}\mathrm{d}z
\end{align}
Using \cref{fact:change-of-variables}, under the change-of-variables $\hat{z} = \Phi(z)$, the symplectic structure changes to
\begin{align}
    \hat{\Omega} &= \frac{1}{2} \mathrm{d}\hat{z}\wedge\mathbb{J}\mathrm{d}z \\
    &= \frac{1}{2} \nabla_z\Phi(z)^\top~\mathrm{d}z \wedge\mathbb{J}\nabla_z\Phi(z)^\top~\mathrm{d}z.
\end{align}
Using \cref{fact:properties-wedge-product},
\begin{align}
    \hat{\Omega} &= \frac{1}{2} ~\mathrm{d}z \wedge\nabla_z\Phi(z) \mathbb{J}\nabla_z\Phi(z)^\top~\mathrm{d}z
\end{align}
Hence we see that $\hat{\Omega}=\Omega$ when $\nabla_z\Phi(z) \mathbb{J}\nabla_z\Phi(z)^\top =\mathbb{J}$, which conforms with the definition of symplecticness.
\end{proof}

\begin{fact}[Time Derivative and Symplecticness]\label{fact:symplectic-time-derivative}
Let $\Phi(\cdot ; t) :\mathrm{T}^*M\to\mathrm{T}^*M$ be a smooth function. Let $\hat{z}_t = \Phi(z;t)$ be a change-of-variables given $z\in\mathrm{T}^*M$ such that $z = \Phi(z;0)$. Let $\hat{\Omega}_t \defeq \frac{1}{2} \mathrm{d}\hat{z}_t \wedge\mathbb{J}~\mathrm{d}\hat{z}_t$. Then $\Phi(\cdot;t)$ is symplectic with respect to $\Omega =\frac{1}{2} \mathrm{d}z\wedge\mathbb{J}~\mathrm{d}z$ if $\frac{\mathrm{d}}{\mathrm{d}t} \hat{\Omega}_t = 0$.
\end{fact}
\begin{proof}
By the fundamental theorem of calculus,
\begin{align}
    \hat{\Omega}_t - \hat{\Omega}_0 = \int_0^t \frac{\mathrm{d}}{\mathrm{d}s} \hat{\Omega}_s \mathrm{d}s.
\end{align}
If $\frac{\mathrm{d}}{\mathrm{d}t} \hat{\Omega}_t = 0$ then
\begin{align}
    \hat{\Omega}_t = \hat{\Omega}_0.
\end{align}
Since $\hat{z}_0 = z$, $\hat{\Omega}_t = \Omega$. The map $\Phi(\cdot;t)$ is symplectic by \cref{fact:symplecticness-and-2-forms}.
\end{proof}

\subsection{Hamiltonian Dynamics}

\begin{fact}[Flow Property (Page 209 in \cite{lee2003introduction})]\label{fact:flow-property}
Let $\Phi(\cdot; t)$ be a vector field flow to time $t$. Vector field flows satisfy the {\it flow property}:
\begin{align}\label{eq:flow-property}
\Phi(\Phi(q, p; t); -t) = (q, p)
\end{align}
or, equivalently,
\begin{align}\label{eq:flow-property-composition}
    \Phi(\cdot;-t)\circ \Phi(\cdot;t) = \text{Id}
\end{align}
\end{fact}

\begin{fact}[Flows of Hamiltonian Vector Fields are Symplectic (Page 185 in \cite{10.5555/1965128})]\label{fact:hamiltonian-symplectic}
Let $\Phi(\cdot; t) : \mathrm{T}^*M\to\mathrm{T}^*M$ be the vector field flow (see \cref{def:vector-field-flow}) to time $t$ of a Hamiltonian vector field $X_H$ (see \cref{def:hamiltonian-vector-field}). Then $\Phi(\cdot; t)$ is symplectic for every $t$.
\end{fact}
\begin{fact}[Composition of Symplectic Maps Form a Group (Page Page 72 in \cite{10.5555/1965128})]\label{fact:symplectic-composition-group}
Let $\Omega$ be a symplectic 2-form on $\mathrm{T}^*M$. The collection of all maps $\Phi : \mathrm{T}^*M\to\mathrm{T}^*M$ such that $\Phi^*\Omega = \Omega$ forms a group under function composition.
\end{fact}

\subsection{Embbeded Geometry}

\begin{definition}[Embedded Cotangent Space]\label{def:embedded-cotangent}
Let $H : \R^m\times\R^m \to\R$ be a smooth function. Let $M$ be manifold that can be embedded in $\R^m$ as the preimage of the zero level set of a constraint function $g : \R^m\to\R^k$; that is, let $M = \set{q\in\R^m : g(q) = 0}$. To view $\mathrm{T}^*M$ as an embedded sub-manifold of $\R^{2m}$ means that $\mathrm{T}^*M$ should be identified with the set
\begin{align}
    \set{(q, p)\in\R^{2m} : g(q) = 0 ~\text{and}~ G(q) \nabla_p H(q, p) = 0}
\end{align}
where $G(q) \in\R^{k\times m}$ is the Jacobian of the constraint function at $q$.
\end{definition}

\begin{fact}[Velocity Constraint]\label{fact:velocity-constraint}
View $\mathrm{T}^*M$ as an embedded sub-manifold of $\R^{2m}$. Given the constraint $g(q) = 0$, we may differentiate this constraint with respect to time to obtain a constraint on the velocity. Namely,
\begin{align}
    \frac{\mathrm{d}}{\mathrm{d}t} g(q) = G(q) \dot{q} = 0.
\end{align}

\end{fact}
\begin{fact}[Velocity and Hamiltonian]\label{fact:velocity-and-hamiltonian}
In Hamiltonian mechanics, $\dot{q} = \nabla_p H(q, p)$. Hence, $G(q) \dot{q} = G(q) \nabla_p H(q, p)= 0$ is the constraint on $p$.
\end{fact}
\begin{fact}[Cotangent Space of Embedded Cotangent Bundle (Page 187 in \cite{leimkuhler_reich_2005})]\label{fact:cotangent-cotangent}
View $\mathrm{T}^*M$ as an embedded sub-manifold of $\R^{2m}$. The embedded cotangent space of $\mathrm{T}^*M$, denoted $\mathrm{T}^*\mathrm{T}^*M$, is a subset of $\mathrm{T}^*\R^{2m}$. Let $\mathrm{d}q_1,\ldots,\mathrm{d}q_m,\mathrm{d}p_1,\ldots,\mathrm{d}p_m \in \mathrm{T}^*\R^{2m}$ be the coordinate 1-forms in the Euclidean space (see \cref{def:coordinate-1-form}). The restriction of these differential 1-forms to $\mathrm{T}^*\mathrm{T}^*M$ implies that they satisfy,
\begin{align}
    G(q) ~\mathrm{d}q &= 0 \\
    f_q(q, p) ~\mathrm{d}q + f_p(q, p)~\mathrm{d}p &= 0,
\end{align}
where $f(q,p) \defeq G(q)\nabla_p H(q, p)$ is the velocity constraint from \cref{fact:velocity-constraint,fact:velocity-and-hamiltonian} and $f_q(q,p)$ (resp. $f_p(q, p)$) represents its Jacobian with respect to $q$ (resp. $p$).
\end{fact}

\begin{fact}[Wedge Product with Lagrange Multipliers Vanish (Page 187 in \cite{leimkuhler_reich_2005})]\label{fact:lagrange-wedge}
Let $q\in\R^n$. Let $g:\R^n\to\R^k$ be the constraint function with Jacobian $G(q)\in\R^{k\times n}$. Suppose $g(q) = 0$. Then for any $\mu\in\R^k$,
\begin{align}
    \mathrm{d}q \wedge \mathrm{d}(G(q)^\top \mu) = 0
\end{align}
\end{fact}
\begin{proof}
We have
\begin{align}
    \mathrm{d}q \wedge \mathrm{d}(G(q)^\top \mu) = \mathrm{d}q \wedge G(q)^\top \mathrm{d}\mu + \sum_{i=1}^k \mathrm{d}q\wedge\mu_i \Gamma_i \mathrm{d}q
\end{align}
where $\Gamma_i$ is the Hessian of the $i^\text{th}$ constraint function. By symmetry of the Hessian and \cref{eq:wedge-annihilation} from \cref{fact:properties-wedge-product}, the second term is zero. The first term is also zero because $g(q) = 0\implies G(q)\mathrm{d}q = 0$ and since $\mathrm{d}q \wedge G(q)^\top \mathrm{d}\mu = G(q)\mathrm{d}q \wedge\mathrm{d}\mu =0$. 
\end{proof}

\subsection{Physics}

\begin{fact}[Total Force]\label{fact:total-force}
The total force acting on an object is the sum of all individual forces.
\end{fact}

\begin{fact}[D'Alembert's Principle]\label{fact:dalemberts-principle}
Constraint forces act in the normal direction to the constraint surface. Given a constraint function $g:\R^m\to\R^k$, constraint forces are therefore represented by $-G(q)^\top\lambda$ for $\lambda\in\R^k$.
\end{fact}
\begin{fact}[Lorentz Force Law]\label{fact:lorentz-force-law}
The force on a particle $q\in\R^3$ under the influence of a magnetic field is given by $m \times \frac{\mathrm{d}}{\mathrm{d}t} q$ where $m\in\R^3$ represents parameters of the magnetic field and $\times$ is the vector cross-product; that is, 
\begin{align}\label{eq:lorentz-force-law}
    m \times \frac{\mathrm{d}}{\mathrm{d}t} q = \begin{pmatrix} 0 & m_1 & -m_2 \\ -m_1 & 0 & m_3 \\ m_2 & -m_3 & 0 \end{pmatrix} \begin{pmatrix} \dot{q}_1 \\ \dot{q}_2\\\dot{q}_3 \end{pmatrix}
\end{align}
\end{fact}

\subsection{Numerical Integration}

\begin{definition}[Order of Integration]\label{def:order-of-integration}
Let $\hat{\Phi}(\cdot; \epsilon, 1) : \mathrm{T}^*M\to\mathrm{T}^*M$ be a single step numerical integrator (\cref{def:numerical-integrator}) for the Hamiltonian vector field flow $\Phi(\cdot; \epsilon) : \mathrm{T}^*M\to\mathrm{T}^*M$ (\cref{def:vector-field-flow}). Then $\hat{\Phi}$ is said to have order $k\in\mathbb{N}$ if for any $(q, p)\in\mathrm{T}^*M$ we have
\begin{align}
    \hat{\Phi}((q, p); \epsilon, 1) - \Phi((q, p); \epsilon) = \mathcal{O}(\epsilon^{k+1})
\end{align}
\end{definition}

\begin{fact}[Symmetric Order of Integration (Page 86 in \cite{leimkuhler_reich_2005})]\label{fact:symmetric-order-integration}
Let $\hat{\Phi}(\cdot; \epsilon, 1) : \mathrm{T}^*M\to\mathrm{T}^*M$ be a single step numerical integrator (\cref{def:numerical-integrator}) for the Hamiltonian vector field flow $\Phi(\cdot; \epsilon) : \mathrm{T}^*M\to\mathrm{T}^*M$ (\cref{def:vector-field-flow}). Suppose further that $\hat{\Phi}$ is a symmetric integrator (\cref{def:symmetric-integrator}). Then the order of $\hat{\Phi}$ is even.
\end{fact}

\newpage
\section{Physical Interpretation of Motion}\label{app:proof-of-lemma-forces}

This result requires \cref{fact:total-force,fact:dalemberts-principle,,fact:lorentz-force-law}.

\ifappendix
\begin{lemma*}
\else
\begin{lemma}[Physical Interpretation of Motion]\label{lem:forces}
\fi
Let $M = \set{q \in \R^3 : g(q) = 0}$ be a manifold such that $G(q)$ has full-rank. Consider a Hamiltonian of the form $H(q, p) = U(q) + \frac{1}{2b}p^\top p$ with $(q,p)\in \mathrm{T}^*M$ and $b\in \R_+$. Then the equations of motion in \cref{eq:embedding-magnetic-velocity,eq:embedding-magnetic-acceleration,eq:embedding-magnetic-constraint} correspond to the motion of a particle, with mass $b$, simultaneously undergoing potential, magnetic, and manifold constraint forces.
\ifappendix
\end{lemma*}
\else
\end{lemma}
\fi

\begin{proof}
It is common to express potential forces as the negative gradient of some function $U:\R^3\to\R$ called the potential function. The Hamiltonian equations of motion for $H(q, p)$ with a magnetic symplectic structure $\Omega_\text{mag}$ are:
\begin{align}
    \dot{q}_t &= \frac{p_t}{b} \\
    \dot{p}_t &= -\nabla_qU(q_t) - \mathrm{L} \frac{p_t}{b} - G(q_t)^\top \lambda \\
    g(q_t) &= 0
\end{align}
Now noting that the momentum variables $p_t = b~\dot{q}_t$ by substitution we obtain,
\begin{align}
    \dot{q}_t &= \dot{q}_t \\
    \underbrace{b~\ddot{q}_t}_{\text{total force}} &= ~~\underbrace{-\nabla_qU(q_t)}_{\text{potential force}}~~ + \underbrace{m\times \dot{q}_t}_{\text{magnetic force}} + ~~\underbrace{G(q_t)^\top (-\lambda)}_{\text{constraint force}}
\end{align}
where, since $\mathrm{L}$ is a skew-symmetric matrix, we have used \cref{fact:lorentz-force-law} to identify
\begin{align}
    -\mathrm{L} = \begin{pmatrix} 0 & m_1 & -m_2 \\ -m_1 & 0 & m_3 \\ m_2 & -m_3 & 0 \end{pmatrix}.
\end{align}
We have used \cref{fact:dalemberts-principle} to identify constraint forces and \cref{fact:total-force} to recognize that the sum of these three forces is the total force acting on the particle. Thus we see, by Newton's second law of motion, that the Hamiltonian equations of motion are equivalent to Newtonian mechanics describing a particle subject to potential, magnetic, and constraint forces.
\end{proof}
\newpage
\section{Embedded Manifold Examples}\label{app:embedded-manifold-examples}

\begin{example}[Euclidean Space]\label{ex:euclidean}
Consider $\mathbb{R}^m$ which may degenerately be regarded as an embedded manifold whose constraint function is $g(q)=0$ for all $q\in\R^m$ (i.e., the euclidean space, unconstrained). The Jacobian of the constraint is the $1\times m$ vector of zeros, which is evidently not full-rank. Nevertheless, continuing the development shows that, for instance, $\mathrm{T}_q\R^m =\mathrm{T}^*_q\R^m= \R^m$ and $\mathrm{T}^*\R^m=\R^m\times\R^m\cong\R^{2m}$.
\end{example}

\begin{example}[The Sphere]
As a second example, consider $\mathbb{S}^2$, the sphere, embedded in $\R^3$ as the preimage of the constraint function $g(q) = q^\top q - 1$ on the zero level set. The Jacobian of the constraint at $q\in\mathbb{S}^2$ is $G(q) = 2q^\top$ which has full-rank as a $1\times 3$ matrix. The tangent space at $q\in \mathbb{S}^2$ is $\mathrm{T}_q\mathbb{S}^2 = \set{\xi \in\R^3 : 2q^\top \xi = 0}$, the set of vectors orthogonal to $q$. Let $\mathrm{R}$ be a $3\times 3$ rotation matrix; an example of a mapping from $\mathbb{S}^2\to\mathbb{S}^2$ is $\Phi(q) = \mathrm{R}q$, the rotation of $q$ by $\mathrm{R}$. In this case, $\mathrm{T}_q\Phi = \mathrm{R}$ so that $(\mathrm{T}_q\Phi)\xi = \mathrm{R}\xi$, the rotation of the tangent vector by $\mathrm{R}$.
\end{example}

\newpage
\section{Comparison of Magnetic Geodesics}\label{app:magnetic-geodesics}

To give intuition for the motion generated by manifold-constrained magnetic Hamiltonian dynamics, we consider the motion of a particle under a Hamiltonian consisting purely of kinetic energy: $H(q, p) = \frac{1}{2} p^\top p$. In the case of canonical dynamics, motion in $q$ generated by this Hamiltonian can be shown to produce geodesic movement on a manifold \citep{10.5555/1965128}; that is, motion for which the particle experiences zero acceleration on the manifold. When a magnetic field is introduced, the resulting motion in $q$ is called a ``magnetic geodesic.'' 

\begin{figure*}[t!]
    \centering
    \begin{subfigure}[b]{0.3\textwidth}
        \centering
        \includegraphics[width=\textwidth]{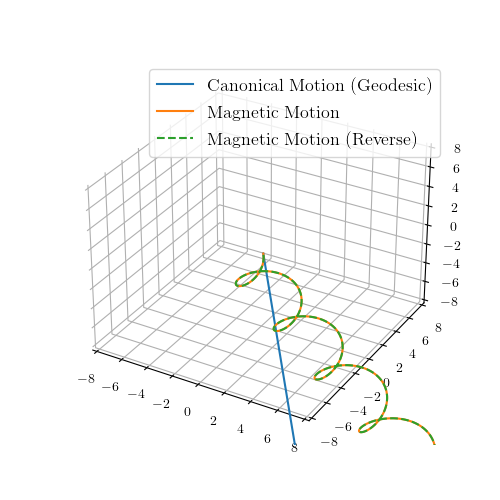}
        \caption{$\R^3$}
    \end{subfigure}
    ~
    \begin{subfigure}[b]{0.3\textwidth}
        \centering
        \includegraphics[width=\textwidth]{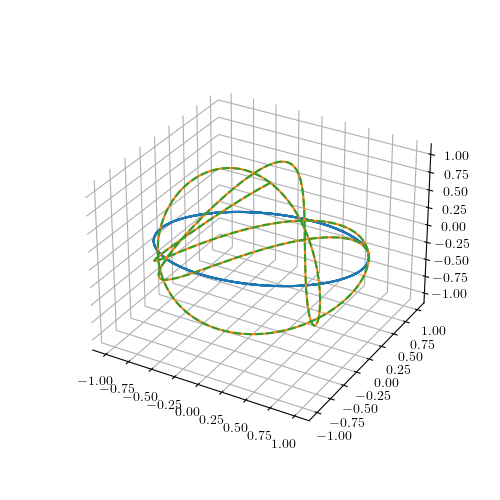}
        \caption{$\mathbb{S}^2$}
    \end{subfigure}
    ~
    \begin{subfigure}[b]{0.3\textwidth}
        \centering
        \includegraphics[width=\textwidth]{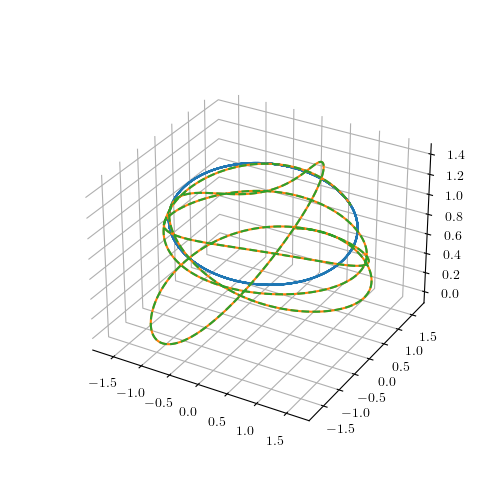}
        \caption{Action of $\text{SO}(3)$}
    \end{subfigure}
    \caption{Visualization canonical and magnetic geodesics for a Hamiltonian consisting purely of kinetic energy. Magnetic trajectories tend to exhibit more unusual behavior. We verify the reversibility of our numerical integrator by integrating the magnetic geodesic forward, then backward, in time. The forward and reverse trajectories trace the same paths because the integrator is reversible.}
    \label{fig:magnetic-geodesic}
\end{figure*}

We visualize the magnetic geodesic for a randomly generated $\mathrm{L}$ in $\R^3$, the sphere, and the special orthogonal group in \cref{fig:magnetic-geodesic}. Whereas the Euclidean geodesic is a straight line, the magnetic geodesic proceeds in a helix. On the sphere, the geodesic corresponds to great circles. The magnetic geodesic on the sphere is much more complicated, visiting many distinct regions of the sphere compared to the usual geodesic which returns to its initial position. We visualize a magnetic geodesic on $\mathrm{SO}(3)$ via its {\it action} on the vector $(1, 1, 1)^\top$. The action of the usual geodesic causes the vector to move about in a circle. The magnetic geodesic yields yields a more unusual and complicated motion of this vector.

We also illustrate that our integrator is reversible under applying a sign flip to the integration step-size $\epsilon\mapsto-\epsilon$. These reverse trajectories have initial condition equal to the terminal condition of the forward trajectory and are integrated for the same number of integration steps with the reversed step-size. We see that in every case the reverse trajectory proceeds backwards along the magnetic geodesic, demonstrating symmetry of the integrator.

\newpage
\section{Proof of \Cref{thm:magnetic-cotangent}}\label{app:proof-magnetic-cotangent}

This result requires \cref{fact:velocity-and-hamiltonian,def:coordinate-1-form,def:embedded-cotangent,fact:cotangent-cotangent,fact:change-of-variables}.

\begin{theorem}\label{thm:magnetic-cotangent}
Let $M = \set{q \in \R^m : g(q) = 0}$ be a connected manifold such that $G(q)$ has full-rank. Let $\text{T}^*M$ be an embedded sub-manifold of $\R^{2m}$ as in \cref{def:cotangent-embedding}. Let $\Omega_\text{mag}$ be the magnetic symplectic structure from \cref{def:magnetic-symplectic-structure} in the ambient Euclidean space $\R^{2m}\cong \R^m\times \R^m$. Let $H(q,p)$ be a smooth Hamiltonian $H:\R^m\times\R^m\to \R$ of the form in \cref{eq:hamiltonian-form}. Let $\Phi_\text{mag}$ be the magnetic vector field flow from \cref{def:magnetic-flow}. Let $(q, p)\in\mathrm{T}^*M$ and $(q_t,p_t)=\Phi_\text{mag}((q, p); t)$. Then the embedded differential one-forms $\mathrm{d}q_t$ and $\mathrm{d}p_t$ respect the manifold constraints such that $(\mathrm{d}q_t,\mathrm{d}p_t)$ are elements of $\mathrm{T}^*_{(q_t,p_t)}\mathrm{T}^*M$.
\end{theorem}

In proving these results we will adopt the shorthand notation $H_{q}$ to denote the partial derivative of $H$ with respect to $q$ regarded as a row vector. The notation $H_{qp}$ denotes the matrix of partial derivatives of $H$ with respect to $q$ and $p$. Other quantities similarly defined. 

We will require preliminary lemmas before proving the theorems.
\begin{lemma}\label{lem:differential-i}
For Hamiltonians of the form in \cref{eq:hamiltonian-form},
\begin{align}
    H_{qp}(q, p) = H_{pq}(q, p) = \mathbf{0}_m
\end{align}
\end{lemma}
\begin{proof}
The Hamiltonian is separable so that the potential energy $U(q)$ is a function of $q$ alone and the kinetic energy $K(p) = \frac{1}{2} p^\top p$ is a function of $p$ alone. Differentiating with respect to $q$ and then with respect to $p$, or with respect to $p$ and then with respect to $q$ causes all terms to vanish.
\end{proof}

\begin{lemma}\label{lem:differential-ii}
Given the constraint $g(q_t) = 0$, the differential $\mathrm{d}q_t$ must obey $G(q_t) ~\mathrm{d}q_t = 0$.
\end{lemma}
\begin{proof}
Applying \cref{fact:change-of-variables}
\begin{align}
    & \mathrm{d}(g(q_t)) = \mathrm{d}0 = 0\\
    \implies& G(q_t) ~\mathrm{d}q_t  =0
\end{align}
\end{proof}

\begin{lemma}\label{lem:differential-iii}
\begin{align}
\mathrm{d} \dot{q}_t &= H_{qp}(q_t, p_t) \mathrm{d}q_t + H_{pp}(q_t, p_t) \mathrm{d}p \\
\mathrm{d} \dot{p}_t &= -H_{qq}(q_t, p_t) \mathrm{d}q_t - H_{qp}(q_t, p_t)\mathrm{d}p_t - \mathrm{L}H_{pq}(q_t, p_t)\mathrm{d}q_t - \mathrm{L}H_{pp}(q_t, p_t)\mathrm{d}p_t - \mathrm{d}(G(q_t)^\top\lambda)
\end{align}
\end{lemma}
\begin{proof}
The equations of motion are
\begin{align}
    \dot{q}_t &= \nabla_p H(q_t, p_t) \\
    \dot{p}_t &= -\nabla_q H(q_t, p_t) - \mathrm{L}\nabla_p H(q_t, p_t) - G(q_t)^\top\lambda \\
    g(q_t) &= 0
\end{align}
Computing the differential yields,
\begin{align}
\mathrm{d} \dot{q}_t &= H_{pq}(q_t, p_t) \mathrm{d}q_t + H_{pp}(q_t, p_t) \mathrm{d}p_t \\
\mathrm{d} \dot{p}_t &= -H_{qq}(q_t, p_t) \mathrm{d}q_t - H_{qp}(q_t, p_t)\mathrm{d}p_t - \mathrm{L}H_{pq}(q_t, p_t)\mathrm{d}q_t - \mathrm{L}H_{pp}(q_t, p_t)\mathrm{d}p_t - \mathrm{d}(G(q_t)^\top\lambda)
\end{align}
\end{proof}

We may now prove \cref{thm:magnetic-cotangent}.

\begin{proof}
We want to show that the magnetic Hamiltonian dynamics
\begin{align}
    \dot{q}_t &= \nabla_p H(q_t, p_t) \label{eq:velocity}\\
    \dot{p}_t &= -\nabla_q H(q_t, p_t) - \mathrm{L} \nabla_pH(q_t, p_t) -G(q)^\top \lambda \\
    0 &= g(q_t) \label{eq:manifold-constraint}
\end{align}
have the property that the embedded differential one-forms $\mathrm{d}q$ and $\mathrm{d}p$ satisfy the manifold constraints such that $(\mathrm{d}q,\mathrm{d}p)$ are elements of the cotangent space of the manifold $\mathrm{T}^*\mathrm{T}^*M$ viewed as an embedded submanifold of $\mathrm{T}^*\R^{2m}$. From \cref{fact:cotangent-cotangent}, this is equivalent to verifying that the solution to the magnetic Hamiltonian dynamics obey:
\begin{align}
    G(q_t) ~\mathrm{d}q_t &= 0 \\
    f_q(q_t, p_t) ~\mathrm{d}q_t + f_p(q_t, p_t)~\mathrm{d}p_t &= 0,
\end{align}
where $f(q,p) \defeq G(q)\nabla_p H(q, p)$. Applying \cref{lem:differential-ii} immediately gives the first condition. The second condition follows from computing the time-derivative of $G(q_t) ~\mathrm{d}q_t = 0$:
\begin{align}
    \frac{\mathrm{d}}{\mathrm{d}t} \left[G(q_t)\mathrm{d}q_t\right] &= \left[\frac{\mathrm{d}}{\mathrm{d}t} G(q_t)\right] \mathrm{d}q_t + G(q_t)\mathrm{d}\dot{q}_t \\
    &= \left[\nabla G(q_t) \cdot \dot{q}_t\right]\mathrm{d}q_t + G(q_t) \mathrm{d}\dot{q}_t \\
    &= 0 \label{eq:differential-zero}
\end{align}
Using \cref{fact:velocity-and-hamiltonian,lem:differential-i}, computing the differentials of $f(q, p)$ yields
\begin{align}
    f_q(q_t, p_t) ~\mathrm{d}q_t &= \left[\nabla G(q_t)\cdot\dot{q}_t + G(q_t) H_{pq}(q_t, p_t)\right] ~\mathrm{d}q_t \\
    &= \left[\nabla G(q_t)\cdot\dot{q}_t \right] ~\mathrm{d}q_t \\
    f_p(q_t, p_t)~\mathrm{d}p_t &= G(q_t) H_{pp}(q_t, p_t)~\mathrm{d}p_t \label{eq:momentum-constraint-partial-p}
\end{align}
Using \cref{lem:differential-iii,lem:differential-i}, the differential of \cref{eq:velocity} gives the relation
\begin{align}
    &\mathrm{d}\dot{q}_t = H_{pq}(q_t, p_t) ~\mathrm{d}q_t + H_{pp}(q_t, p_t)~\mathrm{d}p_t \\
    \implies &\mathrm{d}\dot{q}_t = H_{pp}(q_t, p_t)~\mathrm{d}p_t 
\end{align}
whereupon substitution into \cref{eq:momentum-constraint-partial-p} yields,
\begin{align}
    f_p(q_t, p_t)~\mathrm{d}p_t &= G(q_t)\mathrm{d}\dot{q}_t
\end{align}
Therefore,
\begin{align}
    f_q(q_t, p_t) ~\mathrm{d}q_t + f_p(q_t, p_t)~\mathrm{d}p_t &= \left[\nabla G(q_t)\cdot\dot{q}_t\right]~\mathrm{d}q_t + G(q_t)\mathrm{d}\dot{q}_t \\
    &= 0
\end{align}
from \cref{eq:differential-zero}.
\end{proof}
\newpage
\section{Proof of \Cref{thm:magnetic-conservation}}\label{app:proof-magnetic-conservation}

This result requires \cref{def:coordinate-1-form,def:wedge-product-vectors,fact:lagrange-wedge,fact:properties-wedge-product,fact:2-form-symplectic-structure,def:embedded-cotangent,fact:magnetic-symplectic-structure,fact:velocity-constraint,fact:skew-symmetric-annihilation,fact:velocity-and-hamiltonian,fact:symplectic-time-derivative,lem:differential-iii,fact:flow-property}.

\ifappendix
\begin{theorem*}
\else
\begin{theorem}\label{thm:magnetic-conservation}
\fi
Let $M = \set{q \in \R^m : g(q) = 0}$ be a connected manifold such that $G(q)$ has full-rank. Let $\text{T}^*M$ be an embedded sub-manifold of $\R^{2m}$ as in \cref{def:cotangent-embedding}. Let $\Omega_\text{mag}$ be the magnetic symplectic structure from \cref{def:magnetic-symplectic-structure} in the ambient Euclidean space $\R^{2m}\cong \R^m\times \R^m$. Let $H(q,p)$ be a smooth Hamiltonian $H:\R^m\times\R^m\to \R$ of the form in \cref{eq:hamiltonian-form}. Let $\Phi_\text{mag}$ be the magnetic vector field flow from \cref{def:magnetic-flow}. Then 
\begin{enumerate}[(i)]
\itemsep0em 
\item $\Phi_\text{mag}$ is a symmetric map (\cref{def:symmetric-map}): $\Phi_\text{mag}(\Phi_\text{mag}(q, p; t); -t) = (q, p)$.
\item  $\Phi_\text{mag}(\cdot,\cdot;t)$ is a symplectic transformation (\cref{def:symplectic-transformation}) on $\mathrm{T}^*M$: $\Phi_\text{mag}^*\Omega_\text{mag} = \Omega_\text{mag}$.
\item $\Phi_\text{mag}(\cdot,\cdot;t)$ conserves the Hamiltonian: $H(\Phi_\text{mag}(q, p;t)) = H(q, p)$ for any $(q,p)\in\mathrm{T}^*M$.
\end{enumerate}
\ifappendix
\end{theorem*}
\else
\end{theorem}
\fi

There are three statements. First that the magnetic vector field flow is symmetric, second that it is symplectic, and third that it conserves the Hamiltonian. We will prove the three individually.

\begin{lemma}\label{lem:magnetic-symmetric}
Let $M = \set{q \in \R^m : g(q) = 0}$ be a connected manifold such that $G(q)$ has full-rank. View $\text{T}^*M$ as an embedded sub-manifold of $\R^{2m}$. Let $\Omega_\text{mag}$ be the magnetic symplectic structure in the ambient Euclidean space $\R^{2m}\cong \R^m\times \R^m$. Let $H(q,p)$ be a smooth Hamiltonian $H:\R^m\times\R^m\to \R$ of the form in \cref{eq:hamiltonian-form}. Then $\Phi_\text{mag}$ is symmetric: $\Phi_\text{mag}(\Phi_\text{mag}(q, p; t); -t) = (q, p)$.
\end{lemma}
\begin{proof}
The map $\Phi_\text{mag} :\mathrm{T}^*M\to\mathrm{T}^*M$ is a vector field flow by definition (see \cref{def:magnetic-flow}). By \cref{eq:flow-property} in \cref{fact:flow-property}, it is symmetric.
\end{proof}

\begin{lemma}\label{lem:magnetic-conserve-volume}
Let $M = \set{q \in \R^m : g(q) = 0}$ be a connected manifold such that $G(q)$ has full-rank. View $\text{T}^*M$ as an embedded sub-manifold of $\R^{2m}$. Let $\Omega_\text{mag}$ be the magnetic symplectic structure in the ambient Euclidean space $\R^{2m}\cong \R^m\times \R^m$. Let $H(q,p)$ be a smooth Hamiltonian $H:\R^m\times\R^m\to \R$ of the form in \cref{eq:hamiltonian-form}. Then $\Phi_\text{mag}(\cdot, \cdot;t)$ is a symplectic transformation (\cref{def:symplectic-transformation}) on $\mathrm{T}^*M$ for any $t$.
\end{lemma}
\begin{proof}
We want to show that $\Phi_\text{mag}$ is a symplectic transformation (\cref{def:symplectic-transformation}). By \cref{fact:magnetic-symplectic-structure} the magnetic symplectic structure can be written in terms of the wedge product as
\begin{align}
    \mathrm{d}q \wedge\mathrm{d}p + \frac{1}{2} \mathrm{d}q \wedge \mathrm{L}~\mathrm{d}q
\end{align}
from \cref{eq:magnetic-wedge-product}. Let $(q_t, p_t) = \Phi_\text{mag}(q, p; t)$. Denote $\Omega_\text{mag}^t = \mathrm{d}q_t \wedge\mathrm{d}p_t + \frac{1}{2} \mathrm{d}q_t \wedge \mathrm{L}~\mathrm{d}q_t$. From \cref{fact:symplectic-time-derivative}, $\Phi_\text{mag}$ is symplectic for the magnetic 2-form $\Omega_\text{mag}$ if
\begin{align}
    \frac{\mathrm{d}}{\mathrm{d}t} \Omega_\text{mag}^t = 0.
\end{align}
Hence, our proof strategy will establish $\frac{\mathrm{d}}{\mathrm{d}t} \Omega_\text{mag}^t = 0$ which will imply that $\Phi_\text{mag}$ is symplectic.

We use the differentials computed in \cref{lem:differential-iii}. The notation $H_{qp}$ denotes the matrix of partial derivatives of $H$ with respect to $q$ and $p$. Symmetry of partial derivatives yields $H_{pq} = H_{qp}^\top$. The Hessian matrix with respect to $q$ (resp. $p$) is denoted $H_{qq}$ (resp. $H_{pp}$).

Computing the time derivative of $\Omega_\text{mag}^t$, we have that the magnetic symplectic form is preserved under the solution to the constrained system. 
\begin{align}
    \frac{\mathrm{d}}{\mathrm{d}t} \Omega^t_\text{mag} &= \mathrm{d}\dot{q}_t\wedge \mathrm{d}p_t + \mathrm{d}q_t \wedge\mathrm{d}\dot{p}_t + \frac{1}{2} \mathrm{d}\dot{q}_t \wedge \mathrm{L}\mathrm{d}q_t + \frac{1}{2} \mathrm{d}q_t\wedge \mathrm{L}\mathrm{d}\dot{q}_t \\
    &= \mathrm{d}\dot{q}_t\wedge \mathrm{d}p_t + \mathrm{d}q_t \wedge\mathrm{d}\dot{p}_t + \mathrm{d}\dot{q}_t \wedge \mathrm{L}\mathrm{d}q_t \\
    \begin{split}
    &= H_{pq} \mathrm{d}q_t\wedge \mathrm{d}p_t + H_{pp} \mathrm{d}p_t\wedge \mathrm{d}p_t  \\
    &\qquad -\mathrm{d}q_t \wedge H_{qq} \mathrm{d}q_t - \mathrm{d}q_t \wedge H_{qp}\mathrm{d}p_t - \mathrm{d}q_t \wedge \mathrm{L}H_{pq}\mathrm{d}q_t \\
    &\qquad - \mathrm{d}q_t \wedge \mathrm{L}H_{pp}\mathrm{d}p_t - \mathrm{d}q_t \wedge \mathrm{d}(G(q_t)^\top\lambda) \\
    &\qquad + H_{pq} \mathrm{d}q_t\wedge \mathrm{L}\mathrm{d}q_t + H_{pp} \mathrm{d}p_t\wedge \mathrm{L}\mathrm{d}q_t
    \end{split} \\
    \begin{split}
    &= H_{pq} \mathrm{d}q_t\wedge \mathrm{d}p_t +  \\
    &\qquad - \mathrm{d}q_t \wedge H_{qp}\mathrm{d}p_t - \mathrm{d}q_t \wedge \mathrm{L}H_{pq}\mathrm{d}q_t \\
    &\qquad - \mathrm{d}q_t \wedge \mathrm{L}H_{pp}\mathrm{d}p_t - \mathrm{d}q_t \wedge \mathrm{d}(G(q_t)^\top\lambda) \\
    &\qquad + H_{pq} \mathrm{d}q_t\wedge \mathrm{L}\mathrm{d}q_t + H_{pp} \mathrm{d}p_t\wedge \mathrm{L}\mathrm{d}q_t
    \end{split} \\
    \begin{split}
    &= H_{pq} \mathrm{d}q_t\wedge \mathrm{d}p_t - H_{pq}\mathrm{d}q_t \wedge \mathrm{d}p_t \\
    & \qquad - \mathrm{d}q_t \wedge \mathrm{L}H_{pq}\mathrm{d}q_t - \mathrm{L} H_{pq} \mathrm{d}q_t\wedge \mathrm{d}q_t \\
    &\qquad - \mathrm{d}q_t \wedge \mathrm{L}H_{pp}\mathrm{d}p_t - \mathrm{L} H_{pp} \mathrm{d}p_t\wedge \mathrm{d}q_t  \\
    &\qquad - \mathrm{d}q_t \wedge \mathrm{d}(G(q_t)^\top\lambda)
    \end{split} \\
    &= 0
\end{align}
The final equality comes from manipulations of the wedge product using \cref{fact:properties-wedge-product} and using the fact that $\mathrm{L}$ is a skew-symmetric matrix; in particular, we use \cref{eq:wedge-matrix,,eq:wedge-skew-symmetry,eq:wedge-annihilation}. That $\mathrm{d}q_t \wedge \mathrm{d}(G(q_t)^\top\lambda) = 0$ from \cref{fact:lagrange-wedge} was also used.
\end{proof}

\begin{lemma}\label{lem:magnetic-conserve-hamiltonian}
Let $M = \set{q \in \R^m : g(q) = 0}$ be a connected manifold such that $G(q)$ has full-rank. View $\text{T}^*M$ as an embedded sub-manifold of $\R^{2m}$. Let $\Omega_\text{mag}$ be the magnetic symplectic structure in the ambient Euclidean space $\R^{2m}\cong \R^m\times \R^m$. Let $H(q,p)$ be a smooth Hamiltonian $H:\R^m\times\R^m\to \R$ of the form in \cref{eq:hamiltonian-form}. Then $H(\Phi_\text{mag}(q, p;t)) = H(q, p)$ for any $(q, p)\in\mathrm{T}^*M$ so that the Hamiltonian energy is conserved.
\end{lemma}
\begin{proof}
Let $(q_t, p_t) = \Phi_\text{mag}(q, p; t)$. To prove that the Hamiltonian is conserved, we verify that the time derivative of $H(q_t, p_t)$ equals zero. 
\begin{align}
    \frac{\mathrm{d}}{\mathrm{d}t} H(q_t, p_t) &= \nabla_q H(q_t, p_t)\cdot \dot{q}_t + \nabla_p H(q_t, p_t) \cdot\dot{p} \\
    &= \nabla_q H(q_t, p_t)\cdot \nabla_p H(q_t, p_t) - \nabla_p H(q_t, p_t)\cdot (\nabla_q H(q_t, p_t) + \mathrm{L}\nabla_p H(q_t, p_t) + G(q_t)^\top \lambda) \\
    &= -\nabla_p H(q_t, p_t)\cdot \mathrm{L}\nabla_p H(q_t, p_t) - \nabla_p H(q_t, p_t)\cdot G(q_t)^\top \lambda \\ 
    &= 0
\end{align}
by \cref{fact:skew-symmetric-annihilation} using that $\mathrm{L}$ is skew-symmetric and since $G(q_t)\nabla_p H(q_t, p_t) = G(q_t)p_t = 0$ from \cref{fact:velocity-and-hamiltonian,fact:velocity-constraint}. Therefore, by the Fundamental Theorem of Calculus:
\begin{align}
    H(q_t, p_t) - H(q_0, p_0) = \int_0^t \frac{\mathrm{d}}{\mathrm{d}s} H(q_s, p_s) ~\mathrm{d}s = 0.
\end{align}
Therefore $H(q_t, p_t) = H(q_0, p_0)$ and since $(q_0, p_0) = (q, p)$ we have shown that $H(q_t, p_t) = H(q, p)$.
\end{proof}

We now give a proof of \cref{thm:magnetic-conservation}.
\begin{proof}
Apply \cref{lem:magnetic-symmetric,,lem:magnetic-conserve-volume,lem:magnetic-conserve-hamiltonian}.
\end{proof}
\newpage
\section{Proof of \Cref{lemma:euclidean-symmetric-symplectic}}\label{app:proof-euclidean-symmetric-symplectic}

This result requires \cref{fact:hamiltonian-symplectic,,fact:magnetic-motion,fact:symplectic-composition-group,,fact:flow-property}.

\ifappendix
\begin{lemma*}[Symmetry and Symplecticness of \Cref{alg:euclidean-single-step}]
\else
\begin{lemma}[Symmetry and Symplecticness of \Cref{alg:euclidean-single-step}]\label{lemma:euclidean-symmetric-symplectic}
\fi
The single-step integrator for magnetic dynamics in Euclidean space in \cref{alg:euclidean-single-step} is symmetric and symplectic.
\ifappendix
\end{lemma*}
\else
\end{lemma}
\fi

The magnetic integrator for Euclidean spaces is derived as the symmetric composition of three magnetic Hamiltonian vector field flows (\cref{def:magnetic-flow}). Consider a Hamiltonian of the form $H(q, p) = U(q) + \frac{1}{2} p^\top p$. The integrator is derived from a Strang splitting of the Hamiltonian
\begin{align}
H(q, p) = \underbrace{\frac{1}{2}U(q)}_{H_1(q, p)} + \underbrace{\frac{1}{2} p^\top p}_{H_2(q, p)} + \underbrace{\frac{1}{2} U(q)}_{H_1(q, p)}
\end{align}
The complete algorithm is given in \cref{alg:euclidean-single-step}.

The following lemmas are proved in \cite{pmlr-v70-tripuraneni17a}. They can both be derived from the motion established in \cref{eq:magnetic-equations-of-motion} from \cref{fact:magnetic-motion}.
\begin{lemma}\label{lem:euclidean-split-i}
Let $(q_0, p_0)\in\R^{2m}$. Denote the magnetic vector field flow (\cref{def:magnetic-flow}) to time $\epsilon$ of $H_1$ under a magnetic symplectic structure by $\Phi_1^\epsilon(\cdot, \cdot) :\R^{2m}\to\R^{2m}$. Then
\begin{align}
    \Phi_1^\epsilon(q_0, p_0) = (q_0, p_0 - \epsilon / 2 \cdot \nabla U(q_0)).\label{eq:euclidean-split-i}
\end{align}
\end{lemma}
\begin{lemma}\label{lem:euclidean-split-ii}
Let $(q_0, p_0)\in\R^{2m}$. Denote the magnetic vector field flow (\cref{def:magnetic-flow}) to time $\epsilon$ of $H_2$ under a magnetic symplectic structure by $\Phi_2^\epsilon(\cdot, \cdot; \mathrm{L}) :\R^{2m}\to\R^{2m}$. Then $\Phi_2^\epsilon$ has a closed-form expression given by
\begin{align}\label{eq:euclidean-split-ii}
    (p', q') = \Phi_2^\epsilon(q, p; \mathrm{L})
\end{align}
where
\begin{align}\label{eq:euclidean-split-ii-momentum}
    p' &\defeq \exp(-\epsilon~\mathrm{L}) p \\ \label{eq:euclidean-split-ii-position}
    q' &\defeq q + \begin{pmatrix} U_{\mathbf{D}} & U_\mathbf{0} \end{pmatrix} \begin{pmatrix} \mathbf{D}^{-1} (\exp(\epsilon\mathbf{D}) - \text{Id}) & \mathbf{0} \\ \mathbf{0} & \epsilon \text{Id}\end{pmatrix}\begin{pmatrix} U_{\mathbf{D}} & U_\mathbf{0} \end{pmatrix}^{-1} p 
\end{align}
where
\begin{align}
    -\mathrm{L} = \begin{pmatrix} U_{\mathbf{D}} & U_\mathbf{0} \end{pmatrix} \begin{pmatrix} \mathbf{D} & \mathbf{0} \\ \mathbf{0} & \mathbf{0} \end{pmatrix} \begin{pmatrix} U_{\mathbf{D}} & U_\mathbf{0} \end{pmatrix}^{-1}
\end{align}
is an eigen-decomposition of $-\mathrm{L}$ so that $\mathbf{D}$ is the diagonal matrix of non-zero eigenvalues, $U_{\mathbf{D}}$ is the matrix of eigenvectors for the non-zero eigenvalues, and $U_\mathbf{0}$ is the matrix of eigenvectors for the zero eigenvalues.
\end{lemma}

We will now prove \cref{lemma:euclidean-symmetric-symplectic}. This was already proven in \cite{pmlr-v70-tripuraneni17a}. Here we offer an alternative proof. There are two statements: (i) that the integrator is symmetric and (ii) that the integrator is symplectic. We will prove each individually.

\begin{lemma}\label{lem:euclidean-symmetric}
The single-step integrator in \cref{alg:euclidean-single-step} is symmetric.
\end{lemma}
\begin{proof}
A single step of the {\it numerical} integrator is the symmetric composition of Hamiltonian flows for three sub-Hamiltonians; that is, it is the composition
\begin{align}
\hat{\Phi}(\cdot;\epsilon) \defeq \Phi_1^\epsilon\circ\Phi_2^\epsilon\circ\Phi_1^\epsilon : \R^{2m} \to\R^{2m},
\end{align}
where $\Phi_1^\epsilon$ is the magnetic Hamiltonian vector field flow defined in \cref{lem:euclidean-split-i} and $\Phi_2^\epsilon$ is the magnetic Hamiltonian vector field flow defined in \cref{lem:euclidean-split-ii}. Because the composition is symmetric, it is reversible under negation of the step-size $\epsilon \mapsto -\epsilon$ by the flow property of differential equations from \cref{fact:flow-property} using \cref{eq:flow-property-composition}:
\begin{align}
    \hat{\Phi}(\cdot;-\epsilon) \circ \hat{\Phi}(\cdot;\epsilon) &= \Phi_1^{-\epsilon}\circ\Phi_2^{-\epsilon}\circ\Phi_1^{-\epsilon} \circ \Phi_1^\epsilon\circ\Phi_2^\epsilon\circ\Phi_1^\epsilon \\
    &= \Phi_1^{-\epsilon}\circ\Phi_2^{-\epsilon} \circ\Phi_2^\epsilon\circ\Phi_1^\epsilon \\
    &= \Phi_1^{-\epsilon}\circ\Phi_1^\epsilon \\
    &= \text{Id}.
\end{align}
\end{proof}

\begin{lemma}\label{lem:euclidean-symplectic}
The single-step integrator in \cref{alg:euclidean-single-step} is symplectic.
\end{lemma}
\begin{proof}
A single step of the {\it numerical} integrator is the symmetric composition of Hamiltonian flows for three sub-Hamiltonians; that is, it is the composition
\begin{align}
\hat{\Phi}(\cdot;\epsilon) \defeq \Phi_1^\epsilon\circ\Phi_2^\epsilon\circ\Phi_1^\epsilon.
\end{align}
where $\Phi_1^\epsilon$ is the magnetic Hamiltonian vector field flow defined in \cref{lem:euclidean-split-i} and $\Phi_2^\epsilon$ is the magnetic Hamiltonian vector field flow defined in \cref{lem:euclidean-split-ii}. Hamiltonian flows are symplectic from \cref{fact:hamiltonian-symplectic} and form a group under composition from \cref{fact:symplectic-composition-group}. Therefore, the integrator, which is a composition of three Hamiltonian flows, is symplectic.
\end{proof}

We now give the proof of \cref{lemma:euclidean-symmetric-symplectic}.
\begin{proof}
Apply \cref{lem:euclidean-symmetric,lem:euclidean-symplectic}.
\end{proof}

\newpage
\section{Proof of \Cref{thm:manifold-symmetric-symplectic}}\label{app:proof-manifold-symmetric-symplectic}

This result requires \cref{lemma:euclidean-symmetric-symplectic,fact:lagrange-wedge,def:embedded-cotangent,fact:magnetic-symplectic-structure,fact:symplecticness-and-2-forms,fact:inverse-function-theorem}.

\ifappendix
\begin{theorem*}[Symmetry and Symplecticness of \Cref{alg:manifold-integrator}]
\else
\begin{theorem}[Symmetry and Symplecticness of \Cref{alg:manifold-integrator}]\label{thm:manifold-symmetric-symplectic}
\fi
The integration scheme in \cref{alg:manifold-integrator} is  symplectic and symmetric.
\ifappendix
\end{theorem*}
\else
\end{theorem}
\fi

To prove this theorem, we'll first establish several related lemmas. The first result is a quick verification that the integrator is constrained to the manifold.

\begin{lemma}
Let $M = \set{q \in \R^m : g(q) = 0}$ be a connected manifold such that $G(q)$ has full-rank. Let $\text{T}^*M$ be an embedded sub-manifold of $\R^{2m}$ as in \cref{def:cotangent-embedding}. Let $\mu$ and $\mu'$ be Lagrange multipliers such that \cref{eq:position-constraint,eq:momentum-constraint} are satisfied. Then \cref{alg:manifold-integrator} maps $(q_n, p_n)\in\mathrm{T}^*M$ to $(q_{n+1}, p_{n+1})\in\mathrm{T}^*M$.
\end{lemma}
\begin{proof}
Recall that $M = g^{-1}(0)$ so that $q\in M\iff g(q)=0$. If $\mu$ is a Lagrange multiplier such that \cref{eq:position-constraint} is satisfied, it is immediate that $q_{n+1}\in M$. From \cref{def:cotangent-space}, $p\in \mathrm{T}_q^*M \iff G(q)\nabla_p H(q, p) = 0$. For Hamiltonians in the form of \cref{eq:hamiltonian-form}, $\nabla_p H(q, p) = p$ so that $p\in \mathrm{T}_q^*M \iff G(q)p = 0$. If $\mu'$ is a Lagrange multiplier satisfying \cref{eq:momentum-constraint}, then it is immediate that $p_{n+1}\in\mathrm{T}_{q_{n+1}}^*M$. Thus, by \cref{def:cotangent-embedding}, $(q_{n+1}, p_{n+1})\in\mathrm{T}^*M$.
\end{proof}

\begin{lemma}\label{lem:manifold-integrator-i}
Let $g : \R^m\to\R^k$ be a constraint function with full-rank Jacobian $G : \R^m\to\R^{k\times m}$. Let $q_n\in\R^m$ satisfy $g(q_n) = 0$ and let $p_n\in\R^m$. Let $\bar{p}_{n+1/2} = p_n - \frac{\epsilon}{2} G(q_n)^\top \mu$ for $\mu\in\R^k$. Then,
\begin{align}
\mathrm{d}q_n \wedge\mathrm{d}\bar{p}_{n+1/2} = \mathrm{d}q_n \wedge \mathrm{d}p_n
\end{align}
\end{lemma}
\begin{proof}
By direct calculation using \cref{eq:wedge-linearity} from \cref{fact:properties-wedge-product},
\begin{align}
\mathrm{d}q_n \wedge\mathrm{d}\bar{p}_{n+1/2} &= \mathrm{d}q_n \wedge \mathrm{d}p_n - \mathrm{d}q_n \wedge \mathrm{d}(\frac{\epsilon}{2} G(q_n)^\top \mu) \\
&= \mathrm{d}q_n \wedge \mathrm{d}p_n - \frac{\epsilon}{2} \paren{\mathrm{d}q_n \wedge \mathrm{d}(G(q_n)^\top \mu)} \\
&= \mathrm{d}q_n \wedge \mathrm{d}p_n
\end{align}
by \cref{fact:lagrange-wedge}. \Cref{fact:lagrange-wedge} can be applied since $g(q_n)=0$ and $\mu\in\R^k$ by assumption.

\end{proof}
\begin{corollary}\label{cor:manifold-integrator-i}
For a skew-symmetric matrix $\mathrm{L}\in\text{Skew}(m)$,
\begin{align}
    \mathrm{d}q_n \wedge\mathrm{d}\bar{p}_{n+1/2} + \frac{1}{2} \mathrm{d}q_{n}\wedge \mathrm{L}\mathrm{d}q_{n} = \mathrm{d}q_n \wedge\mathrm{d}p_{n} + \frac{1}{2} \mathrm{d}q_{n}\wedge \mathrm{L}\mathrm{d}q_{n}
\end{align}
\end{corollary}
\begin{proof}
Using \cref{lem:manifold-integrator-i}, add $\frac{1}{2} \mathrm{d}q_{n}\wedge \mathrm{L}\mathrm{d}q_{n}$ to both sides.
\end{proof}

\begin{lemma}\label{lem:manifold-integrator-ii}
Let $(q_n, \bar{p}_{n+1/2})\in\R^{2m}$. Let $\Phi: \R^{2m}\to\R^{2m}$ be a symplectic transformation with respect to the magnetic symplectic form $\Omega_\text{mag}$. If $(q_{n+1}, \bar{p}_{n+1}) = \Phi(q_n, \bar{p}_{n+1/2})$ then,
\begin{align}
    \mathrm{d}q_{n+1}\wedge \mathrm{d}\bar{p}_{n+1} + \frac{1}{2} \mathrm{d}q_{n+1}\wedge \mathrm{L}\mathrm{d}q_{n+1} &= \mathrm{d}q_{n}\wedge \mathrm{d}\bar{p}_{n+1/2} + \frac{1}{2} \mathrm{d}q_{n}\wedge \mathrm{L}\mathrm{d}q_{n}
\end{align}
\end{lemma}
\begin{proof}
Since $\Phi$ is symplectic we have,
\begin{align}
    \Phi^*\Omega_\text{mag} = \Omega_\text{mag}.
\end{align}
Using \cref{eq:magnetic-wedge-product} from \cref{fact:magnetic-symplectic-structure} we express the symplecticness of $\Phi$ using coordinate differential one-forms (\cref{def:coordinate-1-form}):
\begin{align}
    \Phi^*\Omega_\text{mag} = \Omega_\text{mag} \implies \mathrm{d}q_{n+1}\wedge \mathrm{d}\bar{p}_{n+1} + \frac{1}{2} \mathrm{d}q_{n+1}\wedge \mathrm{L}\mathrm{d}q_{n+1} &= \mathrm{d}q_{n}\wedge \mathrm{d}\bar{p}_{n+1/2} + \frac{1}{2} \mathrm{d}q_{n}\wedge \mathrm{L}\mathrm{d}q_{n}
\end{align}
\end{proof}

The statement of the theorem consists of two parts. That the integrator is symplectic and that the integrator is symmetric. We prove each condition individually.
\begin{lemma}\label{lem:manifold-integrator-symplectic}
Let $\mu$ and $\mu'$ be Lagrange multipliers such that \cref{eq:position-constraint,eq:momentum-constraint} are satisfied. The integrator in \cref{alg:manifold-integrator} is symplectic.
\end{lemma}

\begin{proof}
At iteration $n$ of the integrator, assume $(q_n, p_n)\in\mathrm{T}^*M$. The integrator in \cref{alg:manifold-integrator} consists of three steps as follows. Let $(q_n, p_n)\in \mathrm{T}^*M$.
\begin{enumerate}
    \item Set $\bar{p}_{n+1/2} = p_{n} - \frac{\epsilon}{2} G(q_n)^\top \mu$.
    \item Compute $(q_{n+1}, \bar{p}_{n+1})$ using \cref{alg:euclidean-single-step} with input $(q_{n}, \bar{p}_{n+1/2})$, step-size $\epsilon$, and skew-symmetric matrix $\mathrm{L}$.
    \item Set $p_{n+1} = \bar{p}_{n+1} - \frac{\epsilon}{2} G(q_{n+1})^\top \mu'$.
\end{enumerate}
Notice that $\bar{p}_{n+1/2}, \bar{p}_{n+1}\in\R^m$ but that $(q_{n+1}, p_{n+1})\in\mathrm{T}^*M$ by the choice of Lagrange multipliers $\mu$ and $\mu'$.
By \cref{eq:magnetic-wedge-product} from \cref{fact:magnetic-symplectic-structure}, the magnetic symplectic 2-form can be written in terms of wedge products as,
\begin{align}
    \mathrm{d}q_n \wedge\mathrm{d}p_{n} + \frac{1}{2} \mathrm{d}q_{n}\wedge \mathrm{L}\mathrm{d}q_{n}
\end{align}
By \cref{fact:symplecticness-and-2-forms}, it suffices to show that the integrator conserves the symplectic 2-form on $\mathrm{T}^*M$ under the map $(q_n, p_n)\mapsto (q_{n+1}, p_{n+1})$. Therefore, our proof strategy will be to show that
\begin{align}
    \mathrm{d}q_{n+1}\wedge \mathrm{d}p_{n+1} + \frac{1}{2} \mathrm{d}q_{n+1}\wedge \mathrm{L}\mathrm{d}q_{n+1} = \mathrm{d}q_n \wedge\mathrm{d}p_{n} + \frac{1}{2} \mathrm{d}q_{n}\wedge \mathrm{L}\mathrm{d}q_{n}
\end{align}

Since $g(q_n)=0$ by assumption (since $q_n\in M$) and $\mu\in\R^k$, we may apply \cref{cor:manifold-integrator-i} to the first step of the integrator to show that
\begin{align}
    \mathrm{d}q_n \wedge\mathrm{d}\bar{p}_{n+1/2} + \frac{1}{2} \mathrm{d}q_{n}\wedge \mathrm{L}\mathrm{d}q_{n} = \mathrm{d}q_n \wedge\mathrm{d}p_{n} + \frac{1}{2} \mathrm{d}q_{n}\wedge \mathrm{L}\mathrm{d}q_{n}
\end{align}
Applying \cref{lem:manifold-integrator-ii} to $(q_{n+1}, \bar{p}_{n+1})$ in the second step and using the fact that the integrator in \cref{alg:euclidean-single-step} is symplectic by \cref{lemma:euclidean-symmetric-symplectic} shows that,
\begin{align}
    \mathrm{d}q_{n+1}\wedge \mathrm{d}\bar{p}_{n+1} + \frac{1}{2} \mathrm{d}q_{n+1}\wedge \mathrm{L}\mathrm{d}q_{n+1} &= \mathrm{d}q_{n}\wedge \mathrm{d}\bar{p}_{n+1/2} + \frac{1}{2} \mathrm{d}q_{n}\wedge \mathrm{L}\mathrm{d}q_{n} \\
    &= \mathrm{d}q_n \wedge\mathrm{d}p_{n} + \frac{1}{2} \mathrm{d}q_{n}\wedge \mathrm{L}\mathrm{d}q_{n}
\end{align}
Since $g(q_{n+1})=0$ by construction and since $\mu'\in\R^k$, applying \cref{cor:manifold-integrator-i} a second time to the third step yields
\begin{align}
    \mathrm{d}q_{n+1}\wedge \mathrm{d}p_{n+1} + \frac{1}{2} \mathrm{d}q_{n+1}\wedge \mathrm{L}\mathrm{d}q_{n+1} &= \mathrm{d}q_{n+1}\wedge \mathrm{d}\bar{p}_{n+1} + \frac{1}{2} \mathrm{d}q_{n+1}\wedge \mathrm{L}\mathrm{d}q_{n+1} \\
    &= \mathrm{d}q_n \wedge\mathrm{d}p_{n} + \frac{1}{2} \mathrm{d}q_{n}\wedge \mathrm{L}\mathrm{d}q_{n}
\end{align}
This verifies that the symplectic structure $\Omega_\text{can}$ is preserved. Therefore, the integrator is symplectic by \cref{fact:symplecticness-and-2-forms}.
\end{proof}

\begin{lemma}\label{lem:manifold-integrator-symmetric}
Let $\mu$ and $\mu'$ be Lagrange multipliers such that \cref{eq:position-constraint,eq:momentum-constraint} are satisfied. Let $\epsilon$ be the integration step-size. Then the integrator in \cref{alg:manifold-integrator} is symmetric under $\epsilon\mapsto-\epsilon$.
\end{lemma}
\begin{proof}
The integrator in \cref{alg:manifold-integrator} integrator consists of three steps as follows. Let $(q_{n}, p_{n})\in\mathrm{T}^*M$.
\begin{enumerate}
    \item Set $\bar{p}_{n+1/2} = p_{n} - \frac{\epsilon}{2} G(q_n)^\top \mu$.
    \item Compute $(q_{n+1}, \bar{p}_{n+1})$ using \cref{alg:euclidean-single-step} with input $(q_{n}, \bar{p}_{n+1/2})$, step-size $\epsilon$, and skew-symmetric matrix $\mathrm{L}$.
    \item Set $p_{n+1} = \bar{p}_{n+1} - \frac{\epsilon}{2} G(q_{n+1})^\top \mu'$.
\end{enumerate}
Notice that $\bar{p}_{n+1/2}, \bar{p}_{n+1}\in\R^m$ but that $(q_{n+1}, p_{n+1})\in\mathrm{T}^*M$ by the choice of Lagrange multipliers $\mu$ and $\mu'$.
To show that the integration scheme is symmetric, consider beginning from position $(q_{n+1}, p_{n+1})$ and applying the three integration steps with a reversed step-size. In the first step, we obtain the update
\begin{align}
    \bar{p}_{n+1+1/2} &= p_{n+1} + \frac{\epsilon}{2} G(q_{n+1})^\top \mu' \\
    &= \bar{p}_{n+1}
\end{align}
where the last equality derives from rearranging the defining relation in the third step. Since the integrator in \cref{alg:euclidean-single-step} is symmetric by \cref{lemma:euclidean-symmetric-symplectic}, applying the integrator with step-size $-\epsilon$ maps $(q_{n+1}, \bar{p}_{n+1})$ to $(q_n, \bar{p}_{n+1/2})$. The third integration step with Lagrange multiplier $\mu$ yields the update,
\begin{align}
    p_{n+2} &= \bar{p}_{n+1/2} + G(q_n)^\top \mu \\
    &= p_n
\end{align}
By assumption, $(q_n, p_n)\in\mathrm{T}^*M$ so that $g(q_n)=0$ and $G(q_n)p_n = 0$. This completes the reversibility argument.
\end{proof}

We may now prove \cref{thm:manifold-symmetric-symplectic}.
\begin{proof}
Apply \cref{lem:manifold-integrator-symplectic,lem:manifold-integrator-symmetric}
\end{proof}

It remains to be discussed the uniqueness of the Lagrange multipliers $\mu$ and $\mu'$ appearing in \cref{alg:manifold-integrator}. The following result shows that the Lagrange multipliers are uniquely determined when $\epsilon$, the integration step-size, is sufficiently small. The following proof technique is taken from Theorem 4.1 in \cite{geometric-generalization}.
\begin{proposition}
Let $g : \R^m\to\R^k$ be a constraint function with full-rank Jacobian $G : \R^m\to\R^{k\times m}$. Let $(q,p)\in\mathrm{T}^*M$. Define,
\begin{align}
    [(q, p)] \defeq \set{\paren{q, p-G(q)^\top \mu} : \mu\in\R^k}.
\end{align}
Let $\text{Proj}_q : \R^m\times\R^m\to\R^m$ be defined by $\text{Proj}_q(q, p) = q$. Let $\Phi_q^\epsilon\defeq \text{Proj}_q\circ \Phi^\epsilon_1\circ \Phi^\epsilon_2 \circ \Phi_1^\epsilon$ be the projection to the $q$-variables of the approximate integrator of magnetic dynamics in Euclidean space from \cref{app:proof-euclidean-symmetric-symplectic}. Then, for $\epsilon$ sufficiently small, the equation
\begin{align}
    g(\Phi^\epsilon(q^+, p^+)) =0~~~\text{such that}~~~ (q^+, p^+) \in [(q, p)]
\end{align}
has a unique solution in a neighborhood of $\mu=\mathbf{0}$.
\end{proposition}
\begin{proof}
Define the map $F_\epsilon : \R^k \to\R^k$ by
\begin{align}
    F_\epsilon(\mu) = \int_0^1 \frac{\partial}{\partial \epsilon'} \paren{g(\Phi_q^{\epsilon'}(q, p -  G(q)^\top\mu))} \bigg|_{\epsilon\tau} ~\mathrm{d}\tau.
\end{align}
If $\epsilon\neq 0$ then,
\begin{align}
    F_\epsilon(\mu) &= \frac{g(\Phi_q^{\epsilon}(q, p - G(q)^\top\mu)) - g(\Phi_q^{0}(q, p - G(q)^\top\mu))}{\epsilon} \\
    &= \frac{g(\Phi_q^{\epsilon}(q, p -  G(q)^\top\mu))}{\epsilon}.
\end{align}
On the other hand if $\epsilon=0$ then,
\begin{align}
    F_0(\mu) &= G(q) \text{Proj}_q \paren{\frac{\mathrm{d}}{\mathrm{d}\epsilon} (\Phi^\epsilon_1\circ \Phi^\epsilon_2 \circ \Phi_1^\epsilon)(q, p - G(q)^\top\mu)} \\
    &= G(q)(p - G(q)^\top\mu) \\
    &= G(q)p - G(q)G(q)^\top\mu
\end{align}
since $\Phi^\epsilon_1\circ \Phi^\epsilon_2 \circ \Phi_1^\epsilon$ has order greater than one. Note that $\nabla_\mu F_0(\mu)=G(q)G(q)^\top$. Thus, if $G(q)$ has full-rank, then $G(q)G(q)^\top$ is invertible and, by the inverse function theorem (\cref{fact:inverse-function-theorem}), there is a neighborhood $O$ of $\mathbf{0}\in\R^k$ such that $F_0$ is a diffeomorphism of $O$ and $F_0(O)$. Moreover, since $F_\epsilon$ depends smoothly on $\epsilon$, and since $\epsilon\mapsto\text{det}(\nabla_\mu F_\epsilon)$ is continuous, it follows that for sufficiently small $\epsilon$, there exists a neighborhood $O_\epsilon$ of $\mathbf{0}\in\R^k$ such that $F_\epsilon$ is a diffeomorphism of $O_\epsilon$ and $F_\epsilon(O_\epsilon)$. Moreover, since $F_0(\mathbf{0}) = \mathbf{0}$ (since $(q, p)\in\mathrm{T}^*M$), we have that $\mathbf{0}\in F_0(O)$. Therefore, for small enough $\epsilon$, it also follows that $\mathbf{0}\in F_\epsilon(O_\epsilon)$.
\end{proof}

\begin{proposition}
Let $g : \R^m\to\R^k$ be a constraint function with full-rank Jacobian $G : \R^m\to\R^{k\times m}$, let $q_n, p_n\in\R^m$ with $g(q_n) = 0$. Let $\mu'\in\R^k$ be a Lagrange multiplier  chosen such that $G(q_n)\left[p_n - \frac{\epsilon}{2} G(q_n)^\top\mu'\right] = 0$. Then $\mu'$ is uniquely defined.
\end{proposition}
\begin{proof}
The condition $G(q_n)\left[p_n - \frac{\epsilon}{2} G(q_n)^\top\mu'\right] = 0$ can be rearranged as,
\begin{align}
    G(q_n)p_n = \frac{\epsilon}{2} G(q_n)G(q_n)^\top\mu'.
\end{align}
Since $\frac{\epsilon}{2} G(q_n)G(q_n)^\top$ is invertible for non-zero $\epsilon$, $\mu'$ is uniquely determined.
\end{proof}
\newpage
\section{Order of Manifold Integrator}\label{app:manifold-integrator-order}

This result requires \cref{fact:symmetric-order-integration,def:order-of-integration,thm:manifold-symmetric-symplectic}.

\begin{theorem}\label{thm:manifold-second-order}
The integrator in \cref{alg:manifold-integrator} has order (see \cref{def:order-of-integration}) at least two.
\end{theorem}

To prove this result we will first require the following lemma, which was proved in \cite{pmlr-v70-tripuraneni17a}
\begin{lemma}\label{lem:single-step-order}
The single-step subroutine in \cref{alg:euclidean-single-step} has order at least two.
\end{lemma}

\begin{lemma}\label{lem:manifold-first-order}
Let $\mu$ and $\mu'$ be constraint-preserving Lagrange multipliers. The integrator in \cref{alg:manifold-integrator} has order at least one.
\end{lemma}

Before proving \cref{lem:manifold-first-order}, recall that the equations of motion for magnetic Hamiltonian dynamics from \cref{eq:embedding-magnetic-velocity,eq:embedding-magnetic-acceleration,eq:embedding-magnetic-constraint} are
\begin{align}
    \dot{q}_t &= \nabla_p H(q_t, p_t) \\
    \dot{p}_t &= -\nabla_q H(q_t, p_t) - \mathrm{L}\nabla_p H(q_t, p_t) - G(q_t)^\top \lambda \\
    g(q_t) &= 0. 
\end{align}
The equations of motion may be written in matrix form as,
\begin{align}\label{eq:first-order-flow}
    \begin{pmatrix} \dot{q}_t\\\dot{p}_t\end{pmatrix} = \begin{pmatrix} \mathbf{0} & \text{Id} \\ -\text{Id} & -\mathrm{L} \end{pmatrix} \begin{pmatrix} \nabla_q H(q_t, p_t) \\ \nabla_pH(q_t, p_t) \end{pmatrix} - \begin{pmatrix} \mathbf{0} \\ G(q_t)^\top \lambda \end{pmatrix}.
\end{align}
Notice that \cref{eq:first-order-flow} is the first-order term in the Taylor series expansion of the vector field flow in the time variable:
\begin{align}
    \begin{pmatrix} q_\epsilon \\ p_\epsilon \end{pmatrix} = \begin{pmatrix} q_0 \\ p_0\end{pmatrix} + \epsilon \begin{pmatrix} \dot{q}_0\\\dot{p}_0\end{pmatrix} + \mathcal{O}(\epsilon^2).
\end{align}
Therefore, our proof strategy will be to establish that the vector field flow and the numerical integrator agree to first order.

\begin{proof}[Proof of \Cref{lem:manifold-first-order}]
Recall further that the manifold integrator in \cref{alg:manifold-integrator} consists of the following three steps.
\begin{enumerate}
    \item Set $\bar{p}_{n+1/2} = p_{n} - \frac{\epsilon}{2}G(q_n)^\top \mu$.
    \item Compute $(q_{n+1}, \bar{p}_{n+1})$ using \cref{alg:euclidean-single-step} with input $(q_{n}, \bar{p}_{n+1/2})$, step-size $\epsilon$, and skew-symmetric matrix $\mathrm{L}$.
    \item Set $p_{n+1} = \bar{p}_{n+1} - \frac{\epsilon}{2}G(q_{n+1})^\top \mu'$.
\end{enumerate}
From the fact that \cref{alg:euclidean-single-step} is second order from \cref{lem:single-step-order} we have that,
\begin{align}
    \begin{pmatrix} q_{n+1} \\ \bar{p}_{n+1} \end{pmatrix} &= \begin{pmatrix} q_{n} \\ \bar{p}_{n+1/2} \end{pmatrix} + \epsilon \begin{pmatrix} \mathbf{0} & \text{Id} \\ -\text{Id} & -\mathrm{L} \end{pmatrix} \begin{pmatrix} \nabla_q H(q_n, \bar{p}_{n+1/2}) \\ \nabla_pH(q_n, \bar{p}_{n+1/2}) \end{pmatrix} + \mathcal{O}(\epsilon^2) \\
    &= \begin{pmatrix} q_{n} \\ p_{n} - \frac{\epsilon}{2}G(q_n)^\top \mu \end{pmatrix} + \epsilon \begin{pmatrix} \mathbf{0} & \text{Id} \\ -\text{Id} & -\mathrm{L} \end{pmatrix} \begin{pmatrix} \nabla_q H(q_n, p_{n} - \frac{\epsilon}{2}G(q_n)^\top \mu) \\ \nabla_pH(q_n, p_{n} - \frac{\epsilon}{2}G(q_n)^\top \mu) \end{pmatrix} + \mathcal{O}(\epsilon^2) \\
    &= \begin{pmatrix} q_{n} \\ p_{n} \end{pmatrix}  + \epsilon \begin{pmatrix} \mathbf{0} & \text{Id} \\ -\text{Id} & -\mathrm{L} \end{pmatrix} \begin{pmatrix} \nabla_q H(q_n, p_n) \\ \nabla_pH(q_n, p_{n}) \end{pmatrix} - \frac{\epsilon}{2} \begin{pmatrix} \mathbf{0} \\  G(q_n)^\top \mu \end{pmatrix} + \mathcal{O}(\epsilon^2). \label{eq:order-second-step}
\end{align}
Now expanding $G(q_{n+1})$ as a Taylor series in $\epsilon$ shows $G(q_{n+1}) = G(q_n) + \mathcal{O}(\epsilon)$. Therefore,
\begin{align}
    p_{n+1} &= \bar{p}_{n+1} - \frac{\epsilon}{2} G(q_{n+1})^\top \mu' \\\label{eq:order-third-step}
    &= \bar{p}_{n+1} - \frac{\epsilon}{2} G(q_n)^\top \mu' + \mathcal{O}(\epsilon^2).
\end{align}
Combining \cref{eq:order-second-step,eq:order-third-step} yields,
\begin{align}\label{eq:order-combined-steps}
    \begin{pmatrix} q_{n+1} \\ p_{n+1} \end{pmatrix} &= \begin{pmatrix} q_{n} \\ p_{n} \end{pmatrix}  + \epsilon \begin{pmatrix} \mathbf{0} & \text{Id} \\ -\text{Id} & -\mathrm{L} \end{pmatrix} \begin{pmatrix} \nabla_q H(q_n, p_n) \\ \nabla_pH(q_n, p_{n}) \end{pmatrix} - \epsilon \begin{pmatrix} \mathbf{0} \\  G(q_n)^\top \paren{\frac{\mu + \mu'}{2}} \end{pmatrix} + \mathcal{O}(\epsilon^2).
\end{align}
Comparing \cref{eq:order-combined-steps} and \cref{eq:first-order-flow} with $\lambda = \frac{\mu + \mu'}{2}$ shows that the integrator has order at least one.
\end{proof}

\begin{proof}[Proof of \Cref{thm:manifold-second-order}]
From \cref{lem:manifold-first-order} we know that the manifold integrator has order at least one. From \cref{thm:manifold-symmetric-symplectic} we know the manifold integrator is symmetric. However, from \cref{fact:symmetric-order-integration}, symmetric integrators must have even orders. Therefore, \cref{alg:manifold-integrator} has order at least two.
\end{proof}
\newpage
\section{Proof of \Cref{thm:detailed-balance}}\label{app:proof-of-detailed-balance}

\begin{proof}
In this proof, let $z=(q, p)$ and let $H(z) \equiv H(q, p)$. To establish stationarity, it suffices to show that the transition satisfies detailed balance. Let $Z\subset \mathrm{T}^*M$ be a region of the cotangent bundle. Suppose that $Z'$ is the image of $Z$ under $\mathbb{Q}$ when the positive step-size $\epsilon^*$ is randomly chosen. Suppose further that $Z$ is chosen sufficiently small that the value of the Hamiltonian is constant over $Z$ with value $H(Z)$ and over $Z'$ with value $H(Z')$. By virtue of the fact that the integrator is symplectic, we know $\text{Vol}(Z) = \text{Vol}(Z')$. Let $\delta_{\epsilon}(z\to z')$ be the indicator function for the condition that $z$ is transformed to $z'$ under $\mathbb{Q}$ with the integration step-size $\epsilon$. The probability that a randomly generated $z\sim \pi(z)$ will lie in $Z$, that the positive step-size $\epsilon=+\epsilon^*$ is chosen, and that $z$ will subsequently transition from $Z$ to $Z'$ is,
\begin{align}
    &\int_{Z'}\int_Z \frac{\exp(-H(z))}{\mathcal{Z}_H} \cdot \frac{1}{2} \cdot \min\set{1, e^{H(z') - H(z)}}\cdot \delta_{+\epsilon^*}(z\to z') ~\mathrm{d}z\mathrm{d}z' \\
    &\qquad =\frac{\exp(-H(Z))}{\mathcal{Z}_H} \cdot\text{Vol}(Z) \cdot\frac{1}{2} \cdot \min\set{1, e^{H(Z) - H(Z')}} \\
    &\qquad =\frac{\exp(-H(Z'))}{\mathcal{Z}_H} \cdot\text{Vol}(Z') \cdot\frac{1}{2} \cdot \min\set{1, e^{H(Z') - H(Z)}} \\
    &\qquad =\int_{Z}\int_{Z'} \frac{\exp(-H(z'))}{\mathcal{Z}_H} \cdot \frac{1}{2} \cdot \min\set{1, e^{H(z) - H(z')}}\cdot \delta_{-\epsilon^*}(z'\to z) ~\mathrm{d}z'\mathrm{d}z
\end{align}
This last equality is the probability that a randomly generated point $z'\sim \pi(z)$ will lie in $Z'$, that the negative step-size $\epsilon=-\epsilon^*$ is chosen, and that $z'$ will subsequently transition to $Z$. Therefore detailed balance is satisfied, establishing stationarity of $\pi$ for the Markov chain. 

Notice that the random selection of the step-size is necessary for this proof to hold. If $\epsilon$ were fixed (say, $\epsilon=+\epsilon^*$) then there could no guarantee that $Z$ overlaps the image of $Z'$ under $\mathbb{Q}$ (with the positive step-size). In this case, the probability to transition from $Z'$ to $Z$ would be zero making satisfaction of the detailed balance condition impossible.
\end{proof}
\newpage
\section{Symplectic Maps Conserve Volume}\label{app:conservation-of-volume}

This result requires \cref{def:determinant,fact:determinant,def:volume-form,,def:liouville-volume-form,,def:constant-k-form,,def:non-vanishing,,fact:wedge-product-pullback}.

For a closer look at the differential geometry, one might ask, ``In what sense does conservation of the symplectic structure imply conservation of volume?'' 
\begin{theorem}\label{thm:symplectic-volume-preservation}
Transformations that preserve the symplectic structure under pullback preserve the Liouville volume form from \cref{def:liouville-volume-form}.
\end{theorem}
\begin{proof}
The Liouville volume form is defined by,
\begin{align}
    \Lambda \defeq \frac{(-1)^{m(m-1)/2}}{m!} \Omega\wedge \cdots\wedge\Omega
\end{align}
If $\Phi$ is symplectic so that $\Phi^* \Omega = \Omega$, then using \cref{fact:wedge-product-pullback} immediately implies $\Phi^*\Lambda = \Lambda$ so that the volume measure is conserved under $\Phi$.
\end{proof}

\begin{theorem}
Let $\Phi_\text{mag}$ be the magnetic vector field from from \cref{def:magnetic-flow}. Then $\Phi_\text{mag}$ preserves the canonical Liouville volume form $\Lambda_\text{can}$ from \cref{def:liouville-volume-form}.
\end{theorem}

\begin{proof}
From \cref{thm:symplectic-volume-preservation} $\Phi_\text{mag}$ conserves the magnetic Liouville volume form
\begin{align}
    \Lambda_\text{mag} \defeq  \frac{(-1)^{m(m-1)/2}}{m!} \Omega_\text{mag}\wedge \cdots\wedge\Omega_\text{mag}
\end{align}
Now recall \cref{fact:dimension-of-volume-form} which says that the space of volume forms is one-dimensional. Hence any constant (see \cref{def:constant-k-form}), non-vanishing (see \cref{def:non-vanishing}) volume form is proportional to any other constant, non-vanishing volume form. Let $\Lambda_\text{can} = c \cdot\Lambda_\text{mag}$ for some $c\in\R$ with $c\neq 0$. Then,
\begin{align}
    \Phi^*\Lambda_\text{can} &=\Phi^*(c\cdot\Lambda_\text{mag}) \\
    &= c\cdot (\Phi^*\Lambda_\text{mag}) \\
    &= c\cdot \Lambda_\text{mag} \\
    &= \Lambda_\text{can}.
\end{align}
By identification, the determinant from \cref{def:determinant} is $\text{det}(\Phi) = 1$ so that $\Phi$ also conserves volume with respect to $\Lambda_\text{can}$ from \cref{fact:determinant}.
\end{proof}
\newpage
\section{Uniquely Defined Lagrange Multipliers}\label{app:lagrange-multipliers}

\begin{theorem}
Let $M = \set{q \in \R^m : g(q) = 0}$ be a connected manifold such that $G(q)$ has full-rank. Then the Lagrange multipliers $\lambda$ in the equations of motion
\begin{align}
    \frac{\mathrm{d}}{\mathrm{d}t} q_t &= \nabla_p H(q_t, p_t) \\
    \frac{\mathrm{d}}{\mathrm{d}t} p_t &= -\nabla_q H(q_t, p_t) - \mathrm{L}\nabla_p H(q_t, p) - G(q_t)^\top \lambda\label{eq:acceleration} \\
    g(q_t) &= 0
\end{align}
are uniquely defined.
\end{theorem}
\begin{proof}
Write $g(q)$ in terms of the individual constraint functions by identifying $g(q) = (g_1(q),\ldots, g_k(q))$. By definition, $g(q_t) = 0$ along a solution of the equations of motion. Therefore,
\begin{align}
    \frac{\mathrm{d}}{\mathrm{d}t} g(q_t) = G(q_t)~\dot{q}_t = 0.
\end{align}
Differentiating the constraint twice with respect to time yields,
\begin{align}
    \frac{\mathrm{d}^2}{\mathrm{d}t^2} g(q_t) &= \frac{\mathrm{d}}{\mathrm{d}t} G(q_t)~\dot{q}_t \\
    &= \left[\nabla G(q_t) \cdot\dot{q}_t\right]\dot{q}_t + G(q_t)~\ddot{q}_t \label{eq:double-time-derivative} \\ 
    &= 0
\end{align}
From Hamilton's equations of motion for constrained motion with a separable Hamiltonian $H(q, p) = U(q) + \frac{1}{2} p^\top p$ we make the identifications:
\begin{align}
    p_t &\defeq \dot{q}_t
\end{align}
Using the same notation as in \cite{leimkuhler_reich_2005}, we define the $k$-dimensional vector $g_{qq}\langle p_t,p_t\rangle \defeq \left[\nabla G(q_t) \cdot p_t\right] p_t$ whose $i^\text{th}$ component is given by,
\begin{align}
    (g_{qq}\langle p_t,p_t\rangle)_i &\defeq \paren{\left[\nabla G(q_t)\cdot p_t\right] p_t}_i \\
    &= \sum_{i=1}^k p_t^\top (\nabla^2 g_i(q_t)) p_t.
\end{align}
Using the fact that $\ddot{q}_t = -\nabla U(q_t) - \mathrm{L}p_t - G(q_t)^\top \lambda$ from \cref{eq:acceleration} and $\left[\nabla G(q_t) \cdot\dot{q}_t\right]\dot{q}_t = -G(q_t)~\ddot{q}_t$ from \cref{eq:double-time-derivative} we obtain,
\begin{align}
    & G(q_t) \left[-\nabla U(q_t) - \mathrm{L}p_t - G(q_t)^\top \lambda\right] = -g_{qq}\langle p_t,p_t\rangle \\
    \implies& -G(q_t)G(q_t)^\top\lambda = G(q_t)\nabla U(q_t) + G(q_t)\mathrm{L}p_t -g_{qq}\langle p_t,p_t\rangle \\
    \implies& \lambda = -(G(q_t)G(q_t)^\top)^{-1}\left[G(q_t)\nabla U(q_t) + G(q_t)\mathrm{L}p -g_{qq}\langle p_t,p_t\rangle\right].
\end{align}
The matrix $G(q_t)G(q_t)^\top$ is invertible if $G(q_t)$ has full-rank and therefore $\lambda$ will be uniquely defined.
\end{proof}

\newpage
\section{Strang Splitting}\label{app:strang-splitting}

This result requires \cref{fact:hamiltonian-symplectic,fact:symplectic-composition-group}.

Let $H(q, p)$ be a smooth Hamiltonian. The purpose of a numerical integrator is to approximate the Hamiltonian vector field flow (\cref{def:vector-field-flow}) of $H$ to time $t$, denoted $\Phi(\cdot; t)$.
\begin{definition}[Strang Splitting]
Suppose $H(q, p)$ is a Hamiltonian of the form,
\begin{align}
    H(q, p) = H_1(q, p) + \cdots + H_k(q, p)
\end{align}
and that the Hamiltonian vector field flow $\Phi_i$ for each $H_i(q, p)$ has a closed-form expression. The technique known as Strang splitting constructs a numerical integrator of $\Phi$ via the composition
\begin{align}
    \hat{\Phi} = \Phi_1\circ \cdots\circ \Phi_k.
\end{align}
\end{definition}
{\it An integrator derived from Strang splitting is a composition of exact solutions to Hamilton's equations of motion}. This fact makes it easy to show that the {\it integrator} has certain desirable properties. For instance, they are symplectic.
\begin{lemma}
Strang Splitting Integrators are symplectic.
\end{lemma}
\begin{proof}
A composition of Hamiltonian flows is symplectic since each $\Phi_i$ is symplectic from \cref{fact:hamiltonian-symplectic} and the composition of symplectic transformations forms a group from \cref{fact:symplectic-composition-group}.
\end{proof}

The leapfrog integrator can be derived from a Strang splitting argument. Let $H(q, p) = U(q) + \frac{1}{2} p^\top p$ and let $\Omega_\text{can}$ be the symplectic structure (with matrix from \cref{eq:canonical-symplectic-matrix}). Let the splitting of $H$ be
\begin{align}
    H(q, p) = \underbrace{\frac{1}{2} U(q)}_{H_1(q, p)} + \underbrace{\frac{1}{2} p^\top p}_{H_2(q, p)} + \underbrace{\frac{1}{2} U(q)}_{H_1(q, p)}.
\end{align}

\begin{lemma}\label{lem:leapfrog-splitting-i}
The Hamiltonian vector field flow of $H_1$ to time $t$ is
\begin{align}
    (q_0, p_0 - \frac{t}{2} \nabla U(q_0)) = \Phi_1(q_0, p_0; t).
\end{align}
\end{lemma}
\begin{proof}
The equations of motion (\cref{def:hamilton-equations-of-motion}) of $H_1$ are
\begin{align}
    \dot{q} &= 0 \\
    \dot{p} &= -\frac{1}{2} \nabla_qU(q).
\end{align}
Noting that $q$ is constant during the motion, the flow of these equations of motion is seen to have a closed-form expression as
\begin{align}
    q_t &= q_0 + \int_0^t 0 ~\mathrm{d}s = q_0 \\
    p_t &= p_0 - \int_0^t \paren{\frac{1}{2} \nabla_qU(q_0)} ~\mathrm{d}s = p_0 - \frac{t}{2} \nabla U(q_0).
\end{align}
\end{proof}

\begin{lemma}\label{lem:leapfrog-splitting-ii}
The Hamiltonian vector field flow of $H_2$ to time $t$ is
\begin{align}
    (q_0 + tp, p_0) = \Phi_2(q_0, p_0; t).
\end{align}
\end{lemma}
\begin{proof}
The equations of motion of $H_2$ are
\begin{align}
    \dot{q} = p \\
    \dot{p} = 0
\end{align}
Noting that $p$ is constant during the motion, the flow of these equations of motion is seen to have a closed-form expression as
\begin{align}
    q_t &= q_0 + \int_0^t p_0 ~\mathrm{d}s = q_0 + tp_0 \\
    p_t &= p_0 - \int_0^t 0 ~\mathrm{d}s = p_0.
\end{align}
\end{proof}

\begin{theorem}
The leapfrog integrator is the Strang splitting composition $\Phi_1(\cdot; t)\circ \Phi_2(\cdot;t)\circ \Phi_1(\cdot;t)$.
\end{theorem}
\begin{proof}
Let $(q_0, p_0)\in\R^{2m}$. Recall that the leapfrog integrator to time $t$ is defined as the following series of updates. 
\begin{enumerate}
    \item Compute $p_{t/2} = p_0 - \frac{t}{2} \nabla_qU(q_0)$.
    \item Compute $q_t = q_0 + tp_{t/2}$.
    \item Compute $p_t = p_{t/2} - \frac{t}{2} \nabla_q U(q_1)$.
\end{enumerate}
Collapsing these updates into a single statement gives:
\begin{align}\label{eq:leapfrog-collapse-q}
    q_t &= q_0 + t\paren{p_0 - \frac{t}{2} \nabla_qU(q_0)} \\\label{eq:leapfrog-collapse-p}
    p_t &= p_0 - \frac{t}{2} \nabla_qU(q_0) - \frac{t}{2} \nabla_qU\paren{q_0 + t\paren{p_0 - \frac{t}{2} \nabla_qU(q_0)}}
\end{align}

From \cref{lem:leapfrog-splitting-i,lem:leapfrog-splitting-ii}, we have
\begin{align}
    \Phi_2(\cdot;t) \circ \Phi_1(\cdot; t)(q_0, p_0) &= \Phi_2\paren{q_0, p_0 - \frac{t}{2}\nabla_q U(q_0)} \\
    &= \paren{q_0 + t\paren{p_0 - \frac{t}{2}\nabla_q U(q_0)}, p_0 - \frac{t}{2}\nabla_q U(q_0)}.
\end{align}
Therefore,
\begin{align}\label{eq:strang-leapfrog}
    \Phi_1(\cdot; t)\circ\Phi_2(\cdot;t) \circ \Phi_1(\cdot; t)(q_0, p_0) &= \paren{q_0 + t\paren{p_0 - \frac{t}{2}\nabla_q U(q_0)}, p_0 - \frac{t}{2}\nabla_q U(q_0) - \frac{t}{2} \nabla_q U\paren{q_0 + t\paren{p_0 - \frac{t}{2}\nabla_q U(q_0)}}}
\end{align}
One sees by inspection that \cref{eq:strang-leapfrog} has components equal to \cref{eq:leapfrog-collapse-q,eq:leapfrog-collapse-p}.
\end{proof}
\newpage
\section{Observations on Magnetic HMC}

{\bf Specialization to canonical HMC}. When using the choice $\mathrm{L}=\mathbf{0}_m$, one observes that magnetic manifold HMC reduces to canonical HMC wherein Lagrange multipliers are used to enforce manifold constraints. This is because, when $\mathrm{L}=\mathbf{0}_m$, the unconstrained integrator in \cref{alg:euclidean-single-step} reduces to a standard leapfrog step. Note that for the variety of Hamiltonian we have considered, it is not necessary to use an implicitly defined numerical integrator, which was the approach in \cite{pmlr-v22-brubaker12}.

{\bf Ergodicity of the Markov chain}. There exist pathological cases afflicting canonical HMC which cause it to not be ergodic. For instance, for certain choices of step-size and number of steps, the chain may never move from its initial position regardless of the sampled momentum variable. Refer to \cite{10.5555/1162264,livingstone2019} for a discussion. This issue may be averted by combining HMC with a Metropolis-adjusted Langevin diffusion. We note that for $\mathrm{L}=\mathbf{0}_m$, a single step of manifold HMC is equivalent (in the $q$-variable) to a discretization of Langevin diffusion. Therefore, one can obtain an ergodic Markov chain by interspersing single steps of canonical HMC into steps of magnetic manifold HMC. Since both procedures satisfy detailed balance with respect to $\pi(q, p)$, the combination of the two will also satisfy detailed balance.

\end{document}